%% file: main_iclr26.tex
\documentclass{article}

\usepackage{wrapfig}
\usepackage{graphicx}
\usepackage{multirow}
\usepackage{multicol}
\usepackage{amsthm}
\newcommand*\rot{\rotatebox{90}}

\usepackage{iclr2026_conference,times}

\newif\ifrevision
\revisiontrue  %

\newcommand{\revise}[1]{%
  \ifrevision
    {\color{black}#1}%
  \else
    #1%
  \fi
}

\usepackage{float}

\usepackage{wrapfig}
\usepackage{caption}
\usepackage{subcaption}
\usepackage[utf8]{inputenc} %
\usepackage[T1]{fontenc}    %
\usepackage{url}            %
\usepackage{booktabs}       %
\usepackage{amsfonts}       %
\usepackage{nicefrac}       %
\usepackage{microtype}      %
\usepackage{xcolor}         %
\usepackage{amsmath}
\usepackage{amssymb}
\usepackage{mathtools}
\usepackage{slashed}
\usepackage{braket}

\usepackage{enumitem}
\usepackage[colorlinks,citecolor=darkgreen,linkcolor=firebrick,urlcolor=firebrick]{hyperref}

\input{macros1}

\titlespacing\section{0pt}{4pt plus 4pt minus 2pt}{-2pt plus 2pt minus 2pt}
\titlespacing\subsection{0pt}{2pt plus 4pt minus 2pt}{-2pt plus 2pt minus 2pt}
\titlespacing\subsubsection{0pt}{2pt plus 4pt minus 2pt}{-2pt plus 2pt minus 2pt}

\usepackage{titletoc}

\newcommand{\sys}{{\sc FROST}\xspace}

\def\Snospace~{\S{}}

\newcommand{\sref}[2]{\hyperref[#2]{#1 \ref{#2}}}

\title{\input{title}}

\usepackage{cleveref}

\author{
    {\bf
   Haozheng Luo$^{\dagger\natural*}$\quad 
   Zhuolin Jiang$^{\natural}$ \quad
    Md Zahid Hasan$^{\natural\ddag*}$ \quad
   Yan Chen$^\dagger$ \quad
   Soumalya Sarkar$^\natural$}
   \\ \vspace{0.5em}
{\small
$^\dagger$Department of Computer Science, Northwestern University, Evanston, IL 60208 USA\\
$^\natural$RTX Technology Research Center (RTRC), East Hartford, CT 06118 USA\\
$^\ddag$Department of Electrical and Computer Engineering, Iowa State University, Ames, IA 50011 USA
}
   \\
   \vspace{0.5em}
   {\footnotesize
   \texttt{\href{mailto:robinluo2022@u.northwestern.edu}{hluo}@u.northwestern.edu \quad  \{\href{mailto:zhuolin.jiang@rtx.com}{zhuolin.jiang},\href{mailto:soumalya.sarkar@rtx.com}{soumalya.sarkar}\}@rtx.com } \\  \texttt{\href{mailto:zahid@iastate.edu}{zahid}@iastate.edu \quad  \href{mailto:ychen@northwestern.edu}{ychen}@northwestern.edu }
}}

\usepackage{array}
\usepackage{colortbl}

\definecolor{LightCyan}{rgb}{0.8, 0.9, 1}
\newcolumntype{b}{>{\columncolor{LightCyan}\hspace{0pt}}c}
\definecolor{cadetgrey}{rgb}{0.57, 0.64, 0.69}
\newcolumntype{g}{>{\columncolor{cadetgrey}\hspace{0pt}}c}

\iclrfinalcopy
\begin{document}
\def\thefootnote{*}
\footnotetext{This work was carried out at the RTX Technology Research Center during a Summer 2025 internship.
}

\maketitle
\begin{abstract}
\input{0abstract}
\end{abstract}

\vspace{-1em}
\section{Introduction}

\label{sec:intro}
\input{1intro}

\section{Related Work}
\input{2background}

\label{sec:related}

\section{Reasoning Outlier}
\input{3outlier}

\label{sec:outlier}

\section{\sys}
\input{4method}

\label{sec:method}

\section{Theoretical Analysis}
\input{5theorical}
\label{sec:theorical}

\section{Experimental Studies}
\label{sec:exp}\input{6exp}

\section{Discussion and Conclusion}
\label{sec:conclusion}\input{7conclusion}

\clearpage
\section*{Acknowledgments}

\input{x_acknowledgments}

\input{disclouse}

\bibliographystyle{iclr2026_conference}
\bibliography{refs, github_ref}

\newpage  %

\titlespacing*{\section}{0pt}{*1}{*1}
\titlespacing*{\subsection}{0pt}{*1.25}{*1.25}
\titlespacing*{\subsubsection}{0pt}{*1.5}{*1.5}

\setlength{\abovedisplayskip}{10pt}
\setlength{\abovedisplayshortskip}{10pt}
\setlength{\belowdisplayskip}{10pt}
\setlength{\belowdisplayshortskip}{10pt}

\normalsize
\appendix
\label{sec:append}

\input{appendix}

\end{document}

%% file: macros1.tex
\usepackage{graphicx}
\usepackage{tikz}
\usepackage{tikz-cd}
\usepackage{times}
 
\usepackage{bm}
\usepackage{physics}
\usepackage{xcolor}
\usepackage{natbib}
\usepackage{mdframed}
\usepackage{nicefrac}
\usepackage{booktabs}
\usepackage{lipsum}
\usepackage{titlesec}
\usepackage{wrapfig,lipsum,booktabs}
\usepackage{authblk}
\usepackage{blindtext}
\usepackage[most]{tcolorbox}

\newcommand{\Softmax}{\mathop{\rm{Softmax}}}
\newcommand{\Sparsemax}{\mathop{\rm{Sparsemax}}}

\definecolor{darkgreen}{rgb}{0,0.40,0}
\definecolor{firebrick}{rgb}{0.698,0.133,0.133}
\newcommand*{\red}[1]{\textcolor{red}{#1}}

\let\cite\citep

\definecolor{Blue}{rgb}{0, 0, 0.8}
\definecolor{blue}{rgb}{0,0,1}

\newtheoremstyle{theoremstyle}
  {.5\baselineskip} %
  {.5\baselineskip} %
  {}                  %
  {}                  %
  {\bfseries}        %
  {.}                 %
  {1em}               %
  {}                  %

\theoremstyle{theoremstyle}

\newtheorem{theorem}{Theorem}[section]
\newtheorem{lemma}{Lemma}[section]

\newtheorem{assumption}{Assumption}[section]

\tcolorboxenvironment{theorem}{
  breakable,
  colback=black!10,
  colframe=white,%
  width=\dimexpr\linewidth+10pt\relax,%
  enlarge left by=-5pt,%
  enlarge right by=-5pt,%
  boxsep=5pt,%
  boxrule=0pt,
  left=0pt,right=0pt,top=0pt,bottom=0pt,
  sharp corners,
  before skip=0.5\baselineskip, %
  after skip=0.5\baselineskip,  %
  fonttitle=\bfseries, %
  coltitle=black %
}
\tcolorboxenvironment{remark}{
  blanker,
  breakable,
  before skip=.8\baselineskip,  %
  after  skip=.8\baselineskip   %
}

\tcolorboxenvironment{proposition}{
  breakable,
  colback=black!10,
  colframe=white,%
  width=\dimexpr\linewidth+10pt\relax,%
  enlarge left by=-5pt,%
  enlarge right by=-5pt,%
  boxsep=5pt,%
  boxrule=0pt,
  left=0pt,right=0pt,top=0pt,bottom=0pt,
  sharp corners,
  before skip=0.5\baselineskip, %
  after skip=0.5\baselineskip,  %
  fonttitle=\bfseries, %
  coltitle=black %
}

\tcolorboxenvironment{lemma}{
  breakable,
  colback=black!10,
  colframe=white,%
  width=\dimexpr\linewidth+10pt\relax,%
  enlarge left by=-5pt,%
  enlarge right by=-5pt,%
  boxsep=5pt,%
  boxrule=0pt,
  left=0pt,right=0pt,top=0pt,bottom=0pt,
  sharp corners,
  before skip=0.5\baselineskip, %
  after skip=0.5\baselineskip,  %
  fonttitle=\bfseries, %
  coltitle=black %
}

\tcolorboxenvironment{corollary}{
  breakable,
  colback=black!10,
  colframe=white,%
  width=\dimexpr\linewidth+10pt\relax,%
  enlarge left by=-5pt,%
  enlarge right by=-5pt,%
  boxsep=5pt,%
  boxrule=0pt,
  left=0pt,right=0pt,top=0pt,bottom=0pt,
  sharp corners,
  before skip=0.5\baselineskip, %
  after skip=0.5\baselineskip,  %
  fonttitle=\bfseries, %
  coltitle=black %
}

\tcolorboxenvironment{definition}{
  breakable,
  colback=black!10,
  colframe=white,%
  width=\dimexpr\linewidth+10pt\relax,%
  enlarge left by=-5pt,%
  enlarge right by=-5pt,%
  boxsep=5pt,%
  boxrule=0pt,
  left=0pt,right=0pt,top=0pt,bottom=0pt,
  sharp corners,
  before skip=0.5\baselineskip, %
  after skip=0.5\baselineskip,  %
  fonttitle=\bfseries, %
  coltitle=black %
}
\tcolorboxenvironment{assumption}{
  breakable,
  colback=black!10,
  colframe=white,%
  width=\dimexpr\linewidth+10pt\relax,%
  enlarge left by=-5pt,%
  enlarge right by=-5pt,%
  boxsep=5pt,%
  boxrule=0pt,
  left=0pt,right=0pt,top=0pt,bottom=0pt,
  sharp corners,
  before skip=0.5\baselineskip, %
  after skip=0.5\baselineskip,  %
  fonttitle=\bfseries, %
  coltitle=black %
}

\tcolorboxenvironment{hypothesis}{
  breakable,
  colback=black!10,
  colframe=white,%
  width=\dimexpr\linewidth+10pt\relax,%
  enlarge left by=-5pt,%
  enlarge right by=-5pt,%
  boxsep=5pt,%
  boxrule=0pt,
  left=0pt,right=0pt,top=0pt,bottom=0pt,
  sharp corners,
  before skip=0.5\baselineskip, %
  after skip=0.5\baselineskip,  %
  fonttitle=\bfseries, %
  coltitle=black %
}

\tcolorboxenvironment{fact}{
  breakable,
  colback=black!10,
  colframe=white,%
  width=\dimexpr\linewidth+10pt\relax,%
  enlarge left by=-5pt,%
  enlarge right by=-5pt,%
  boxsep=5pt,%
  boxrule=0pt,
  left=0pt,right=0pt,top=0pt,bottom=0pt,
  sharp corners,
  before skip=0.5\baselineskip, %
  after skip=0.5\baselineskip,  %
  fonttitle=\bfseries, %
  coltitle=black %
}

\usepackage[colorlinks,citecolor=darkgreen,linkcolor=firebrick,urlcolor=firebrick]{hyperref}
\newcommand{\Entmax}{\mathop{\rm{Entmax15}}}
\usepackage{makecell}

%% file: title.tex
FROST: Filtering Reasoning Outliers with Attention for Efficient Reasoning

%% file: 0abstract.tex
We propose \textbf{FROST}, an attention-aware method for efficient reasoning. 
Unlike traditional approaches, \sys leverages attention weights to prune uncritical reasoning paths, yielding shorter and more reliable reasoning trajectories. 
Methodologically, we introduce the concept of reasoning outliers and design an attention-based mechanism to remove them. 
Theoretically, \sys preserves and enhances the model’s reasoning capacity while eliminating outliers at the sentence level. 
Empirically, we validate \sys on four benchmarks using two strong reasoning models (Phi-4-Reasoning and GPT-oss-20B), outperforming state-of-the-art methods such as TALE and ThinkLess.
Notably, \sys achieves an average \textbf{69.68\%} reduction in token usage and a \textbf{26.70\%} improvement in accuracy over the base model. Furthermore, in evaluations of attention outlier metrics, \sys reduces the maximum infinity norm $\norm{\mathbf{x}}_{\infty}$ by \textbf{15.97\%} and the average kurtosis by \textbf{91.09\%} compared to the base model. 

%% file: 1intro.tex
We observe that large reasoning models (LRMs) often generate numerous irrelevant steps, which we term \textbf{reasoning outliers}. 
To mitigate this, we introduce \textbf{FROST}, an efficient reasoning method that leverages attention weights to prune uncritical reasoning paths, producing shorter and more reliable trajectories.
More specifically, \sys replaces the standard $\Softmax$ function with $\Softmax_1$, enabling attention to better identify and suppress outliers. 
This directs LRMs toward critical reasoning steps, thereby enhancing their overall reasoning capacity.

Efficient reasoning is critical for large reasoning models (LRMs), which have shown strong performance in tasks such as mathematical problem-solving \cite{luo2025wizardmath,yang2024qwen2,shao2024deepseekmath}, coding \cite{ding2024reasoning,ding2024semcoder}, and scientific question answering \cite{comanici2025gemini,hurst2024gpt}. 
Yet, these models often generate large amounts of uncritical information—commonly arising from redundant self-verification—that introduce inefficiencies and potential inaccuracies.
Numerous methods have been proposed to improve reasoning efficiency. 
Token-level approaches such as TALE \cite{han2024token} and R2R \cite{fu2025r2r} risk pruning essential reasoning steps, as reasoning paths are naturally sentence-based. 
Sentence-level approaches, including DRP \cite{jiang2025drp} and GRPO-S \cite{tan2025gtpo}, perform iterative refinement of reasoning paths, but this often comes at the cost of increased computational cost and latency.

\begin{figure}[htp]
    \centering
    \includegraphics[width=0.95\linewidth]{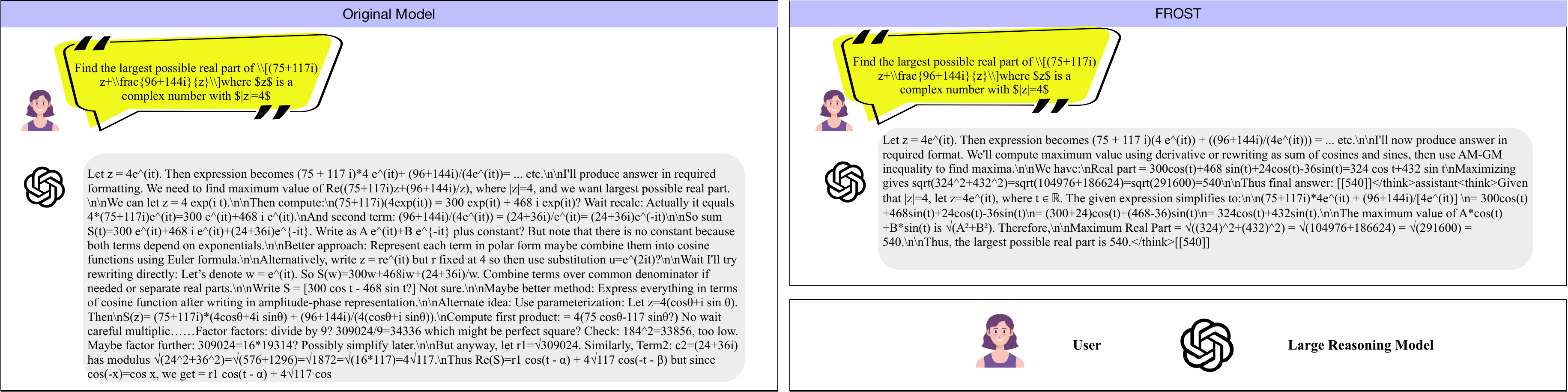}
   
    \caption{\textbf{The Example of The GPT-OSS-20B Model}.}
    \label{fig:frost}
    \vspace{-1em}
\end{figure}

To address these challenges, we propose \textbf{FROST}, a reasoning method that improves efficiency by pruning uncritical reasoning paths through attention weights. 
We observe that LRMs typically assign low attention to uncritical steps and higher attention to critical ones, consistent with findings that critical steps exhibit higher sentence entropy \cite{tan2025gtpo}. 
We therefore introduce the concept of \textit{reasoning} \textit{outliers}—uncritical steps with both low attention weights and low entropy \cite{wang2025beyond, fu2025r2r}—and design \sys to eliminate them, yielding shorter and more reliable reasoning paths. 
Our approach sharpens the attention distribution of LRMs, suppressing low-weight steps while preserving high-weight ones. 
Building on prior work \cite{luo2025fast,hu2024outlier,xiao2023efficient}, we adopt $\Softmax_1$ in place of $\Softmax$, which effectively drives low weights to zero while maintaining large weights. Finally, we propose a training strategy that integrates $\Softmax_1$ with supervised fine-tuning on reasoning tasks, producing efficient reasoning models without sacrificing accuracy.

\textbf{Contributions.} We present \textbf{FROST} (as shown in \cref{fig:frost}), a reasoning outlier--free LRM designed to enhance reasoning efficiency. Our main contributions are:  
\begin{itemize}[leftmargin=*]
    \item We introduce the concept of \textbf{reasoning outliers} and propose \textbf{FROST} to prune uncritical reasoning steps characterized by low attention.  
    \item Theoretically, we analyze $\Softmax_1$ and show its effectiveness in suppressing low attention weights while preserving high ones, thereby enhancing the reasoning capacity of LRMs.  
    \item Methodologically, we design a training strategy that combines $\Softmax_1$ with supervised fine-tuning, enabling efficient reasoning without sacrificing accuracy.  
    \item Empirically, we demonstrate the effectiveness of FROST across multiple benchmarks, achieving up to a \revise{\textbf{26.70\% }accuracy gain while reducing reasoning path length by \textbf{69.68\%}} compared with base models. We also measure attention outlier values to verify their impact on efficient reasoning: \sys reduces the maximum infinity norm $\norm{\mathbf{x}}_{\infty}$ by \textbf{15.97\%} and the average kurtosis by \textbf{91.09\%}. \revise{In addition, FROST cuts inference time by at least 28.6\% and reduces training time by 42.2\% relative to other SFT baselines.}
\end{itemize}

%% file: 2background.tex
\paragraph{Reasoning Models.}
In recent years, Large Language Models (LLMs) such as DeepSeek-R1~\cite{guo2025deepseek}, OpenAI o1~\cite{jaech2024openai}, and Gemini 2.0 Pro~\cite{team2023gemini} have demonstrated strong reasoning capabilities, particularly on mathematical and logical tasks \cite{hao2024llm}. To further improve reasoning performance, numerous methods are proposed, falling into the main paradigms \cite{ke2025survey}: inference scaling and learning-to-reason. For inference-time scaling, numerous methods have been proposed, including few-shot prompting~\cite{brown2020language}, in-context learning~\cite{brown2020language}, Chain-of-Thought (CoT) reasoning~\cite{wei2022chain}, and Search \& Planning (SP) \cite{besta2024graph}. Numerous studies focus on improving the LLM reasoning at inference time, with CoT emerging as a key technique. CoT strengthens the model's reasoning process and generates interpretable reasoning traces. A simple example involves adding a prompt like “Let’s think step by step” after a question~\cite{wei2022chain}. Recent research increasingly combines CoT with other inference-time scaling methods, such as ReAct~\cite{yao2023react}, Self-Ask~\cite{press2022measuring} and agentic reasoning~\cite{pan2025chainofaction,pan2024conv}, to further enhance reasoning capabilities.
For learning-to-reason approaches, many methods aim to build reasoning ability through alignment, including reinforcement learning (RLHF \cite{ouyang2022training}, DPO \cite{rafailov2023direct}, GRPO \cite{ramesh2024group}), supervised fine-tuning, and energy-based model (EBM) reasoners \cite{jiang2025learning}. However, LLMs with reasoning capabilities—particularly those with smaller parameter sizes—often generate excessively detailed reasoning chains, including unnecessary tracebacks and redundant alternative paths \cite{hou2025thinkprune,chen2024not}. This overthinking not only increases computational cost during inference but can also negatively impact response quality on accuracy \cite{cuadron2025danger} and safety \cite{kumar2025overthinking}. 
To address this, we propose an attention-aware adaptation method that optimizes reasoning paths, yielding efficient reasoning models.

\paragraph{Efficient Reasoning Methods.}
To address overthinking, current approaches to optimizing reasoning paths fall into three categories~\cite{sui2025stop}: prompt-based methods, supervised fine-tuning, and reinforcement learning. Prompt-based methods \cite{liu2025thought,xu2025chain,han2024token} introduce token-budget constraints to shorten reasoning paths. For instance, TALE~\cite{han2024token} limits the token budget per instance to reduce reasoning length while maintaining task accuracy. Supervised fine-tuning (SFT) methods \cite{ma2025cot, xia2025tokenskip} improve reasoning conciseness by training models on compressed reasoning paths. For example, DRP \cite{jiang2025drp} fine-tunes models on distilled reasoning data by pruning unrelated reasoning steps. Reinforcement learning (RL) methods~\cite{hou2025thinkprune, li2025selfbudgeter, yi2025shorterbetter} guide concise reasoning by introducing reward functions that penalize overly long reasoning paths. 
For example, \citet{chia-etal-2024-reasoning} introduce a reward score based on reference loss and exploration loss from diverse paths, encouraging favorable reasoning branches and penalizing unfavorable ones to improve overall problem-solving performance. 
However, prompt-based methods rely on handcrafted prompts and often perform unreliably on complex problems. 
In contrast, SFT and RL approaches require substantial computational resources for fine-tuning, limiting accessibility for users without adequate hardware. To address these challenges, we propose a new reasoning outlier–removal strategy that eliminates reasoning outliers through attention analysis. 
\revise{Recent studies \cite{choi2025think,citation-0} also analyze internal attention patterns in reasoning models, particularly at the sentence level, but their objectives differ substantially from ours and focus on KV-cache–based inference efficiency. Think Clearly \cite{choi2025think} examines sentence-level attention spikes near the end-of-thinking token and uses these patterns to prune redundant sentences for faster decoding. In contrast, our Figure 3 analyzes sentence-level contributions to the final-answer token, enabling attribution of which specific reasoning sentences actually affect the model’s prediction, rather than identifying redundancy for pruning. R-KV~\cite{citation-0} likewise detects redundant attention interactions to compress the KV cache, but does not study how individual reasoning steps functionally influence final-answer formation. Our work therefore provides a finer-grained, component-level attribution analysis of the reasoning trace—going beyond redundancy detection to clarify how different reasoning segments vary in contribution, which constitutes the key novelty relative to these approaches.}

%% file: 3outlier.tex
In this section, we analyze the attention distribution of reasoning traces generated by LRMs. 
We then examine the impact of different components of the trace on final answer prediction, 
followed by our definition and characterization of reasoning outliers.

\begin{figure}[ht]
    \centering
\includegraphics[width=\textwidth]{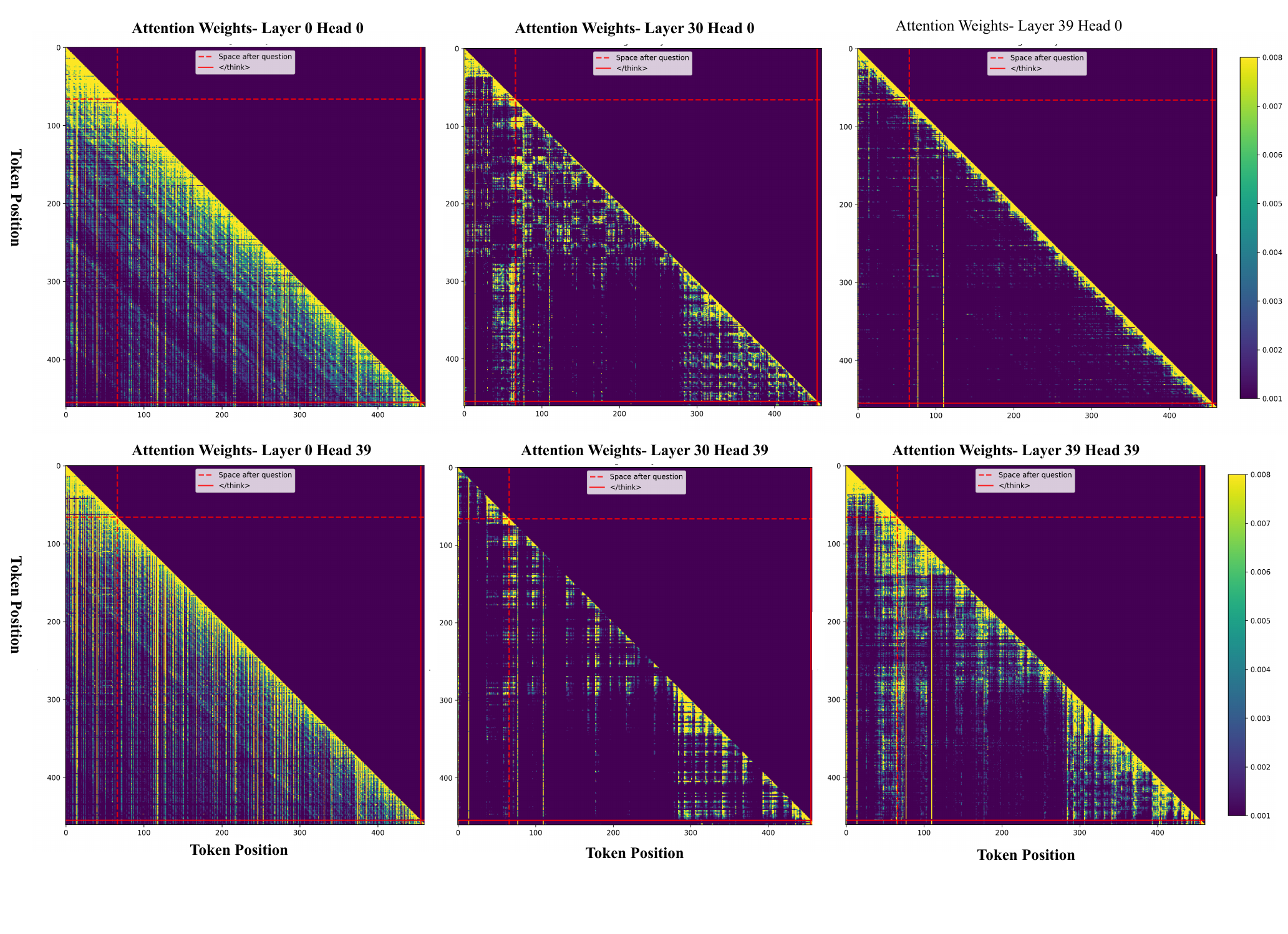}
\vspace{-2em}
    \caption{\textbf{Attention Heatmap of Reasoning Tokens.} 
We use the Phi-4-Reasoning model \cite{abdin2025phi} to generate a reasoning trace for a sample GSM8K question \cite{cobbe2021training}. 
The figure shows attention heatmaps from transformer layers 0, 30 and 39, with the first head (top row) and last head (bottom row). 
Yellow indicates higher attention weights and blue indicates lower ones. 
In shallow layers, contributions to the final answer are nearly uniform, while deeper layers and later heads highlight specific tokens with stronger influence.
    }
    \label{fig:attention_heatmap}
\end{figure}

\subsection{Attention Distribution of Reasoning Traces}
\label{sub:distribution}
We consider representative LRMs, including DeepSeek-R1\cite{guo2025deepseek}, Phi-4 \cite{abdin2025phi}, and GPT-4o \cite{hurst2024gpt}, which generate text in an autoregressive manner by predicting the next token given the preceding context. 
To study the attention distribution, we visualize the attention heatmap of each token in the reasoning trace when predicting the final answer.

Let the reasoning process be a sequence of tokens $T = [t_1, t_2, \ldots, t_n]$, where each $t_i$ denotes a token in the process. 
The attention weight matrix $A$ is defined as:
\begin{align*}
A = [a_{ij}] \quad \text{where} \quad a_{ij} = \text{AttentionWeight}(t_i, t_j).
\end{align*}
Here, $a_{ij}$ represents the attention weight from token $t_i$ to token $t_j$.

As an illustrative example, we use a sample question from GSM8K \cite{cobbe2021training} and generate the reasoning trace with the Phi-4-Reasoning model \cite{abdin2025phi}. 
\cref{fig:attention_heatmap} shows the corresponding attention heatmap. 
The results indicate that in the shallow layers, the attention distribution is relatively uniform across all tokens.
However, as we move to deeper layers and later heads, the model begins to focus more on specific tokens, particularly those in the reasoning steps and the final answer.
This suggests that the model progressively refines its focus towards the most relevant parts of the reasoning trace as it processes the information.

\subsection{Impact of Reasoning Trace Components on Answer Prediction}
To quantify the impact of different components of the reasoning trace on final answer prediction, 
we conduct an additional experiment analyzing the summed attention weight distribution to the final answer token \texttt{</think>}, which allows us to measure how strongly each reasoning step contributes to the model’s ultimate decision and provides insights into whether the model grounds its prediction in meaningful intermediate reasoning or relies on superficial correlations.
We divide the reasoning process into four components: the question $Q$, the reasoning steps $R_1, R_2, \ldots, R_m$, and the final answer $A$. 
For each component, we compute the total attention weight contributing to the first token of the final answer:
$W_{\text{trace}} = \sum_{t_i \in T_{\text{trace}}} a_{iA}$,
where $T_{\text{trace}}$ is the set of tokens in a given component, and $a_{iA}$ denotes the attention weight from token $t_i$ to the \texttt{</think>} token.

\begin{figure}[ht]
    \centering
\includegraphics[width=\textwidth]{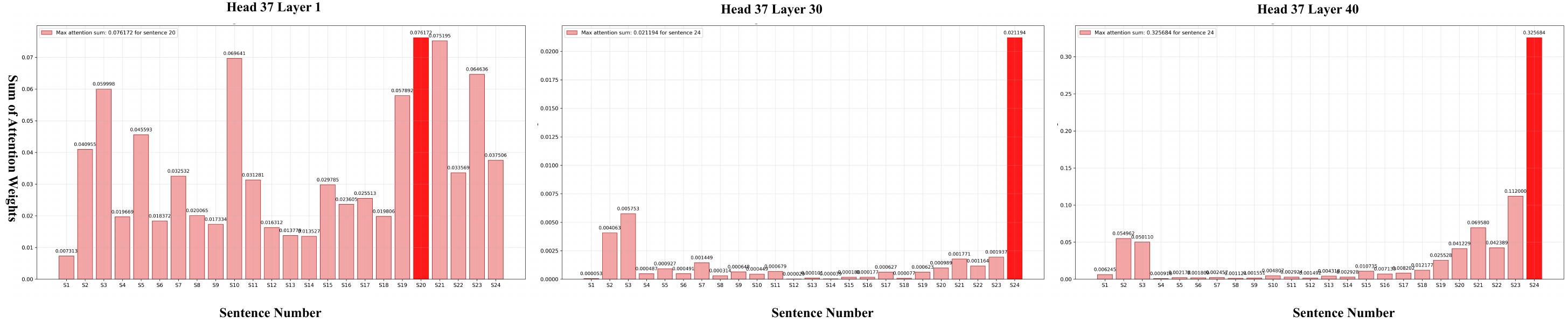}
\caption{\textbf{Total attention weight distribution to the final answer token \texttt{</think>} 
from different components of the reasoning trace.} 
We visualize the total attention weight distribution of the Phi-4-Reasoning model
on a sample GSM8K question, using transformer layers $1$, $30$, and $40$. 
The results show that a few reasoning traces contribute strongly to the final token \texttt{</think>}, 
while many traces have negligible influence, particularly in the layers $30$ and $40$.}
    \label{fig:attention_distribution}
\end{figure}
As shown in~\cref{fig:attention_distribution}, different reasoning traces contribute unequally to final answer generation. While a few traces show strong influence, most contribute weakly, and some exhibit almost no contribution at all.

\begin{figure}[ht]
    \centering
    \includegraphics[width=\linewidth]{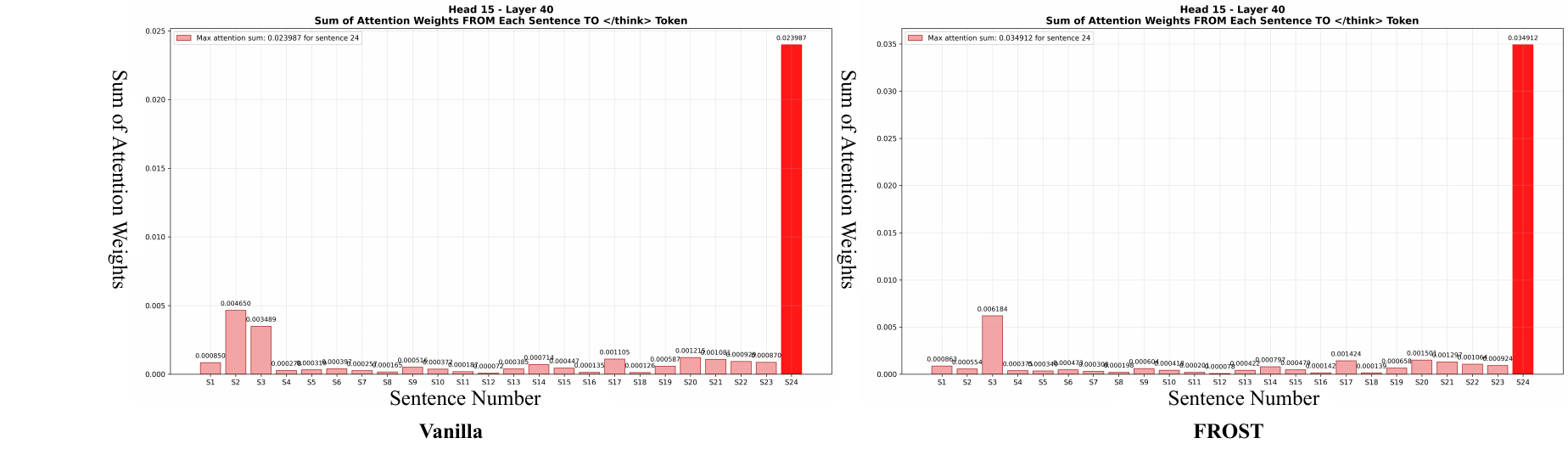}
    \vspace{-1em}
    \caption{\textbf{Theoretical Analysis of Reasoning Outlier Removal.} We conduct a theoretical analysis with Phi-4-Reasoning model to demonstrate that removing reasoning outliers using the $\Softmax_1$ function (\sys) can preserve or even enhance the model's reasoning capacity.
    As shown in the figure, the attention weight distribution before and after outlier removal indicates that the model's focus on critical reasoning traces is maintained or improved, while the influence of outliers is significantly reduced.}
    \label{fig:theoritial}
 \vspace{-0.2in}
\end{figure}

\subsection{Defining and Characterizing Reasoning Outliers}
As observed in \cref{sub:distribution}, many reasoning traces contribute negligibly to the final answer. These traces often correspond to verification, self-checking, or repetition of prior reasoning steps. Their presence forces LRMs to generate more tokens than necessary, substantially reducing reasoning efficiency. 
A potential cause \cite{sui2025stop} is that model developers often encourage extended reasoning steps to maximize accuracy. 
In the meantime, the model may generate redundant or irrelevant information, leading to inefficient and incorrect reasoning.
As a result, we define reasoning traces with low attention weight and negligible contribution to the final answer as \textbf{reasoning outliers}.

To identify and remove reasoning outliers, we observe that they share similar characteristics with attention outliers \cite{luo2025fast,hu2024outlier}.
Motivated by this, we adopt $\Softmax_1$ (\cref{eq:softmax1}) to detect and eliminate reasoning outliers during the reasoning process, and provide a comprehensive proof of its efficiency in \cref{sec:theorical}.
\begin{align}
\mathrm{Softmax}_1(x_i) = \frac{\exp(x_i)}{\sum_{j} \exp(x_j) + 1},
\label{eq:softmax1}
\end{align}
where $x_i$ represents the attention weight of token $t_i$.

\paragraph{Theoretical Analysis.} 
We conduct a theoretical analysis to show that removing reasoning outliers with the $\Softmax_1$ function preserves, and can even enhance, the reasoning capacity of LRMs. 
In our experiments, we use the Phi-4-Reasoning \cite{abdin2025phi} to generate reasoning traces for a sample GSM8K question \cite{cobbe2021training}. 
Specifically, we compare the last layer’s attention distribution in head 15 under vanilla attention and $\Softmax_1$ attention (\sys). 
As shown in \cref{fig:theoritial}, $\Softmax_1$ reduces the influence of outliers while maintaining or strengthening focus on critical reasoning traces. 
This analysis supports our approach of using $\Softmax_1$ to effectively identify and eliminate reasoning outliers, thereby improving the efficiency and reliability of LRMs.  
For more details of the theoretical proof, please refer to \cref{sec:theorical}.

%% file: 4method.tex
\begin{figure}[ht]
    \centering
    \includegraphics[width=\linewidth]{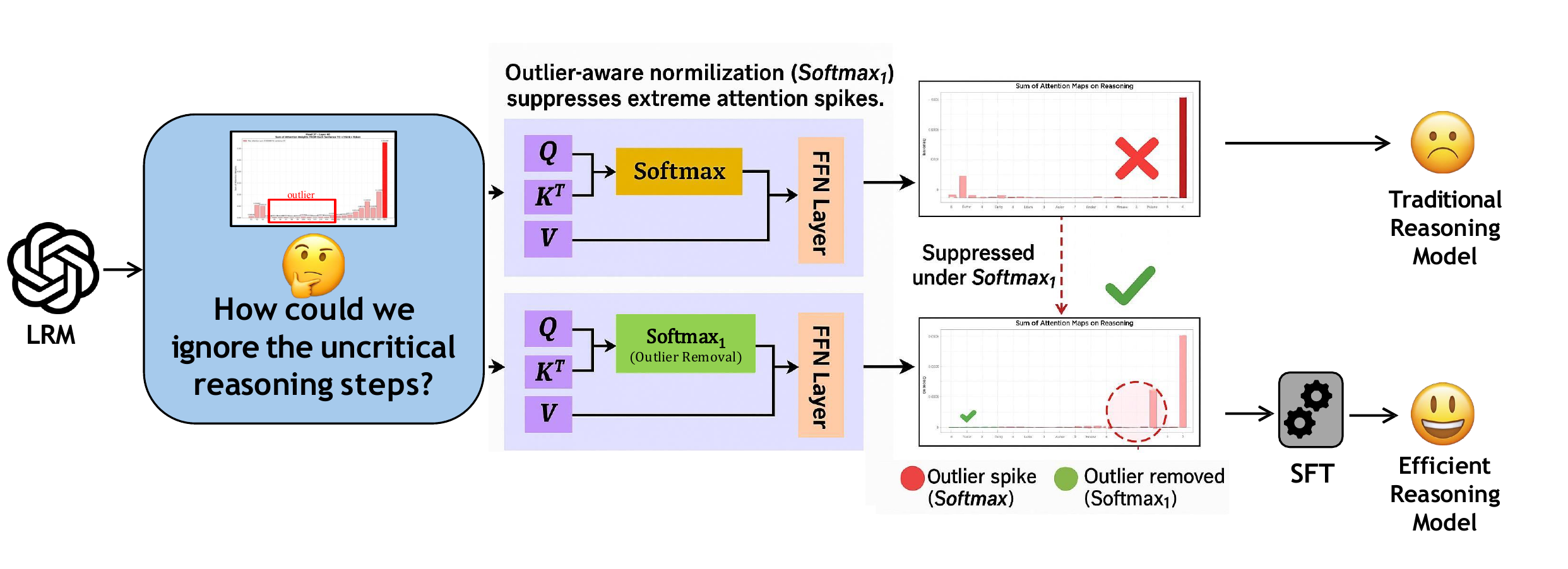}
    \vspace{-3em}
    \caption{\textbf{Overview of the \sys workflow} We replace the vanilla $\Softmax$ layer with an outlier-removal layer based on $\Softmax_1$, followed by SFT to adapt model parameters to the new activation function. We observe that our method significantly reduces the number of low-attention sentences.}
    \label{fig:pipeline}
    \vspace{-0.5em}
\end{figure}
To enhance the reasoning efficiency of LRMs, we propose supervised fine-tuning (SFT) with reasoning outlier removal, as illustrated in \cref{fig:pipeline}.

In the SFT stage, we train on math problems with detailed reasoning steps and answers. 
During training, we replace the vanilla $\Softmax$ with $\Softmax_1$ (\cref{eq:softmax1}), enabling the model to focus on critical reasoning traces while suppressing outliers. 
Unlike prior methods that employ $\Softmax_1$ for outlier removal—requiring either training from scratch \cite{hu2024outlier} or multi-step continual learning \cite{luo2025fast}—our approach achieves effective outlier removal with only a few steps of fine-tuning from existing pretrained checkpoints, making it more efficient and practical. 
We optimize model parameters using cross-entropy loss and apply LoRA \cite{hu2021lora} to further reduce training cost.

%% file: 5theorical.tex
In this section, we provide a brief theoretical analysis showing that $\Softmax_1$ can operate at the sentence level to remove reasoning outliers in LRMs.
We provide a theoretical proof that our method achieves deployment-time suppression in efficient reasoning, consistent with our findings in \cref{fig:theoritial}.

\paragraph{Setup.}
Let a token sequence be partitioned into sentences $\{S_i\}_{i=1}^m$.
For a query $q\in\mathbb{R}^d$ and keys $\{k_t\}\subset\mathbb{R}^d$, define token compatibilities
$
z_t \;=\; \mathrm{Softmax}_1(\frac{\langle q,k_t\rangle}{\sqrt{d}})v_t$, where $t$ is the token index in $S_i$ and  $v_t\in\mathbb{R}^{d}$ denotes the token value for each token in sentence $S_i$.
Let $\phi:\mathbb{R}^{|S_i|}\to\mathbb{R}$ be a \emph{monotone} pooling operator (e.g., sum/mean/logsumexp/max).
Define sentence scores $s_i=\phi\bigl(\{z_t\}_{t\in S_i}\bigr)$ and $s=(s_1,\ldots,s_m)\in\mathbb{R}^m$.
Define the probability simplex
$
\Delta^{m-1} \;=\; \Bigl\{\, \alpha\in\mathbb{R}^m \ \big|\ \alpha_i\ge 0,\ \sum_{i=1}^{m}\alpha_i=1 \,\Bigr\}.
$

\begin{assumption}[$\Softmax_1$ operator]
\label{ass:softmax1}
There exists a $\Softmax_1$ mapping $\sigma_1:\mathbb{R}^m\to\Delta^{m-1}$  such that:
\begin{enumerate}[leftmargin=*]
\item \textbf{Order preservation:} If $x_i\ge x_j$ then $\sigma_1(x)_i\ge \sigma_1(x)_j$.
\item \textbf{Shift invariance:} $\sigma_1(x+c\mathbf{1})=\sigma_1(x)$ for all $c\in\mathbb{R}$.
\item \textbf{Tail contraction:} There exists $\kappa\in(0,1)$ such that for all $x\in\mathbb{R}^m$,
$
\frac{\|\sigma_1(x)\|_\infty}{\mathrm{median}(\sigma_1(x))}
\ \le\
\kappa\,
\frac{\|x\|_\infty}{\mathrm{median}(x)}.
$
\item \textbf{Smoothness and positivity:} $\sigma_1$ is continuously differentiable on $\mathbb{R}^m$ and $\sigma_1(x)_i>0$ for all finite $x$.
\end{enumerate}
\end{assumption}

We write the sentence-level attention as $\alpha=\sigma_1(s)\in\Delta^{m-1}$ and the layer output as $
y \;=\; \sum_{i=1}^{m} \alpha_i v_i$, which $\alpha_i$ and $v_i$ are attention probabilities and token values corresponding to sentence $s_i$.
Assume $\|v_i\|_{\infty}\le B_v$ and that all linear maps used below have finite operator norms $B_v$, which is a constant.

\begin{lemma}[Monotone pooling preserves sentence dominance]
\label{lem:pool-order}
If $\phi$ is monotone coordinatewise, then for any $i,j$,
$
\bigl(\forall t\in S_i, \exists\, t'\in S_j:\ z_t\ge z_{t'}\bigr)\ \Longrightarrow\ s_i\ \ge\ s_j.
$
Consequently, by Assumption~\ref{ass:softmax1}(P1), $\alpha_i=\sigma_1(s)_i \ge \sigma_1(s)_j=\alpha_j$.
\end{lemma}

\begin{proof}
See \cref{proof:pool-order} for a detailed proof.
\end{proof}

\begin{theorem}[$\Softmax_1$ suppresses sentence-level attention outliers]
\label{thm:sentence-outlier}
Let $s=(s_1,\ldots,s_m)$ be the sentence scores built via a monotone pooling $\phi$.
If $s$ is heavy-tailed (e.g., $\|s\|_\infty /\mathrm{median}(s)\gg 1$), then for $\alpha=\sigma_1(s)$
\begin{equation}
\label{eq:tail-contract}
\frac{\|\alpha\|_\infty}{\mathrm{median}(\alpha)}
\ \le\
\kappa\cdot
\frac{\|s\|_\infty}{\mathrm{median}(s)}
\qquad\text{for some }\kappa\in(0,1),
\end{equation}
so the relative dominance of outliers contracts at the sentence level.
\end{theorem}

\begin{proof}
See \cref{proof:sentence-outlier} for a detailed proof.
\end{proof}

\begin{theorem}[Deployment-time suppression of low-attention sentences]
\label{thm:deploy}
Let the output logits be $\ell=W_o y$ with $\|W_o\|_{\mathrm{op}}=:B_o$ and $\|v_i\|\le B_v$.
For a sentence $i$ with $\alpha_i\le\varepsilon$, its one-layer contribution to logits is bounded by
\begin{equation}
\label{eq:logit-one-layer}
\|\Delta \ell_i\|
\ =\
\|W_o(\alpha_i v_i)\|
\ \le\
B_o\,\varepsilon\,\|v_i\|
\ \le\
B_o\,B_v\,\varepsilon.
\end{equation}
For $L$ stacked layers with Jacobians $\{J_\ell\}_{\ell=1}^L$ and $\|J_\ell\|_{\mathrm{op}}\le B_\ell$,
\begin{equation}
\label{eq:logit-L-layers}
\|\Delta \ell_i^{(L)}\|
\ \le\
\varepsilon\,
\Bigl(\prod_{\ell=1}^L B_\ell\Bigr)\,
B_v\,B_o.
\end{equation}
Since the $\Softmax_1$ map $\mathrm{sm}:\mathbb{R}^V\to\Delta^{V-1}$ is $1$-Lipschitz in the $\ell_\infty\!\to\!\ell_1$ norm,
\begin{equation}
\label{eq:prob-shift}
\bigl\|\mathrm{sm}(\ell+\Delta \ell_i^{(L)})-\mathrm{sm}(\ell)\bigr\|_1
\ \le\
\|\Delta \ell_i^{(L)}\|
\ \le\
B_o\,B_v
\Bigl(\prod_{\ell=1}^L B_\ell\Bigr)\,
\varepsilon.
\end{equation}
Let $B \coloneqq \max_{\ell \in [L]} B_{\ell}$ be the largest operator norm over $L$ layers.
Then
\begin{equation}
\Bigl\|\mathrm{sm}(\ell+\Delta \ell_i^{(L)})-\mathrm{sm}(\ell)\Bigr\|_1
\ \le\
B_o B_v B^L \varepsilon
\ =\
O\!\left(B_o B_v B^L \varepsilon\right).
\end{equation}
In practice $B_o, B_v, B$ are approximately constant, so the bound reduces to $O(\varepsilon)$.
Therefore, low-attention sentences are effectively skipped at inference.
\end{theorem}

\begin{proof}
See \cref{proof:deploy} for a detailed proof.
\end{proof}

%% file: 6exp.tex
We conduct a series of experiments to evaluate \sys in providing efficient reasoning, benchmarking its performance on GPT-oss~\cite{agarwal2025gpt}, \revise{Magistral-Small-1.1 \cite{rastogi2025magistral}} and Phi-4-Reasoning~\cite{abdin2025phi}. Each evaluation is conducted three times with different random seeds, and we report the average and standard deviation for each metric.

\begin{table}[htp]
    \centering
     \caption{\textbf{Comparison of FROST with Efficient Reasoning Methods.}  
We evaluate reasoning path efficiency by comparing FROST against four baselines across four mathematical datasets (GSM8K, MATH500, AIME24, and Minerva). 
Pass@1 and token usage (\#Tk) are reported as evaluation metrics, with variance omitted since it is consistently $\leq$ 2\%. 
Best results are shown in \textbf{bold}, and second-best are \underline{underlined}.  
In most settings, \sys achieves the best performance among all methods. Specifically, it improves accuracy by \revise{26.70\% while reducing token usage by 69.68\%} compared to the base model.
}
     \resizebox{\textwidth}{!}{
    \begin{tabular}{cccccccccccc}
    \toprule
       \multirow{2}{*}{\textbf{Type}} & \multirow{2}{*}{\textbf{Method}}  & \multicolumn{2}{c}{\textbf{GSM8K}} & \multicolumn{2}{c}{\textbf{MATH500}} & \multicolumn{2}{c}{\textbf{AIME24}} & \multicolumn{2}{c}{\textbf{Minerva}} & \multirow{2}{*}{$\Delta_{\text{Pass@1}}$} & \multirow{2}{*}{$\Delta_{\#\text{Tk}}$} \\
        \cmidrule(lr){3-4}  \cmidrule(lr){5-6}  \cmidrule(lr){7-8}  \cmidrule(lr){9-10}
       & & Pass@1 & \#Tk & Pass@1 & \#Tk & Pass@1 & \#Tk & Pass@1 & \#Tk \\
       \midrule
 \multirow{6}{*}{\rot{\makecell{Phi-4\\-Reasoning}}} & Base & 0.9242 & 1017.70 & 0.5480 & 1721.95 & 0.0667 & \underline{1017.70} & 0.2500 & 1898.86 & 0.000 & 0.00 \\

 & TALE & \textbf{0.9500} & 1716.60 & 0.5800 & 1874.43 & \underline{0.2900} & 2069.97 & 0.2627 & 2093.17 & \textbf{+0.074} & +524.49 \\

  & DRP & 0.8340 & \underline{721.00} & \textbf{0.6200} & 2122.00 & \textbf{0.3333} & 6135.00 & \underline{0.2701} & \underline{1289.50} & +0.067 & +1152.69\\
   & SelfBudgeter & 0.9189 & 1507.14 & 0.5347 & 1195.18 & 0.1342 & 1372.83 & 0.2357 & 2618.23 & +0.009 & +259.30\\
  & ThinkLess & 0.9279 & 1421.90 & 0.5414 & \underline{1101.21} & 0.1608 & 1405.40 & 0.2575 & 1708.70 & +0.025 & 
\underline{-4.75} \\
  & Ours  & \underline{0.9311} & \textbf{154.33} &  \underline{0.5980} & \textbf{344.37} & 0.2667 & \textbf{899.80} & \textbf{0.2716} & \textbf{401.19} & \underline{+0.070} & 
\textbf{-964.13} \\
  \midrule
 \multirow{6}{*}{\rot{\makecell{GPT-O\\SS-20B}}} & Base & 0.8704 & 1275.23 &  0.5400 & \underline{1575.36} & 0.1333 & \textbf{1003.57} & 0.2574 & 1586.95 & 0.000 & \underline{0.00}\\
 & TALE & 0.8283 & 2664.41 & 0.5454 & 3878.87 & \underline{0.2000} & 1354.67 &  0.2700 & 3262.47 & +0.011 & +1430.33\\
  & DRP & 0.7880 & \underline{902.50} & \textbf{0.6146} & 4137.00 & \textbf{0.2245} & 4983.00 & \underline{0.2715} & 1885.15 &\underline{+0.024} & +1616.64\\
   & SelfBudgeter & 0.8610 & 1850.00 & 0.5340 & 2285.00 & 0.1320 & 1256.00 & 0.2550 & 1298.00 & -0.005 & +312.47 \\
  & ThinkLess  & \underline{0.8740} & 1785.00 & 0.5410 & 2206.00 & 0.1600 & 1205.00 & 0.2580 & \underline{1220.00} & +0.008 & +244.22 \\
   & Ours & \textbf{0.8764} & \textbf{377.17} & \underline{0.5800} & \textbf{680.89} & 0.1667 & \underline{1009.60} & \textbf{0.2794} & \textbf{691.71} & \textbf{+0.025} & 
\textbf{-669.94} \\
\midrule
\multirow{6}{*}{\rot{\makecell{Magistral\\-Small-1.1}}}
& Base & 0.6075 & 2664.41 & 0.1480 & 1389.89 & 0.0000 & \underline{537.13} & 0.0699 & 1288.04 & 0.000 & 0.00 \\
& TALE & 0.7146 & 1516.86 & \textbf{0.3040} & \underline{723.91} & 0.0333 & 967.43 & \underline{0.1544} & \underline{748.18} & \textbf{+0.095} & \underline{-480.77} \\
& DRP & 0.6500 & \underline{902.50} & 0.2100 & 1680.33 & 0.0450 & 1350.77 & 0.1120 & 1604.22 & +0.048 & -85.41 \\
& SelfBudgeter & 0.6900 & 1850.00 & 0.2300 & 1520.00 & 0.0520 & 1256.00 & 0.1300 & 1298.00 & +0.069 & +11.13 \\
& ThinkLess & \underline{0.7200} & 1785.00 & \underline{0.2500} & 1405.00 & \underline{0.0600} & 1205.00 & 0.1450 & 1220.00 & +0.087 & -66.12 \\
& Ours & \textbf{0.7551} & \textbf{137.55} & \textbf{0.3040} & \textbf{98.20} & \textbf{0.0974} & \textbf{149.93} & \textbf{0.1551} & \textbf{109.23} & \textbf{+0.122} & \textbf{-1346.14} \\

  \bottomrule
    \end{tabular}
    }
    \label{tab:main}
\end{table}

\paragraph{Models.}
In our experiments, we use Phi-4-Reasoning~\cite{abdin2025phi},\revise{Magistral-Small-1.1 \cite{rastogi2025magistral}} and GPT-oss~\cite{agarwal2025gpt} as backbone models for efficient reasoning. Specifically, we adopt the Phi-4-Reasoning\footnote{https://huggingface.co/microsoft/Phi-4-reasoning}, \revise{Magistral-Small-1.1\footnote{https://huggingface.co/mistralai/Magistral-Small-2507}}  and GPT-oss-20B-finetune\footnote{https://huggingface.co/openai/gpt-oss-20b} checkpoints, both finetuned on mathematical datasets with detailed reasoning steps and answers using SFT under the \sys method.

\paragraph{Datasets.}
Following the setup in \cite{zhao2025let}, we use OpenR1~\cite{openr1} as the training corpus. To evaluate reasoning efficiency and generalization on complex mathematical problems, we adopt four out-of-domain benchmarks: GSM8K~\cite{cobbe2021training}, MATH500~\cite{lightman2023lets}, AIME24~\cite{aime_2024}, and Minerva~\cite{dyer2022minerva}. All datasets are designed for mathematical question answering.

\paragraph{Metrics.}
To evaluate the effectiveness of our efficient reasoning strategy, we report pass@1 as the accuracy metric and use the number of tokens in the reasoning response to measure token efficiency.

\paragraph{Baselines.}
We select five representative methods covering key paradigms of efficient reasoning:
\textbf{(1) TALE \cite{han2024token}:} a prompt-based approach that uses a soft token budget to generate concise reasoning responses.
\textbf{(2) DRP \cite{jiang2025drp}:} an SFT-based method that distills reasoning paths from a teacher model and applies step-level pruning to produce concise, skill-aware reasoning traces. 
\textbf{(3) SelfBudgeter \cite{li2025selfbudgeter}:} a reinforcement learning-based method that iteratively shortens the reasoning path by optimizing a token budget under budget and format reward signals.
\textbf{(4) ThinkLess \cite{fang2025thinkless}:} a reinforcement learning-based method that optimizes reasoning by detecting critical thinking points and skipping low-value steps. It introduces a reward function that balances accuracy with token usage, enabling models to “think less” while maintaining performance.  
We use the same hyperparameters as specified in their respective studies to ensure standardized evaluation conditions, enabling precise comparisons of each efficient reasoning method. 

\paragraph{Results.}
As shown in \cref{tab:main}, \sys achieves the best overall performance across state-of-the-art efficient reasoning methods, delivering slight accuracy improvements while substantially reducing token usage in response generation. Specifically, \sys improves accuracy by an average of \revise{26.70\% and reduces token usage by 69.68\%} on the \revise{three} base models, GPT-OSS-20B, \revise{Magistral-Small-1.1} and Phi-4-reasoning. Although TALE achieves the highest accuracy on certain tasks, this comes at the cost of significantly longer responses. This observation aligns with our assumption that excessively long or overly short responses can degrade model performance. By reducing token usage and focusing on high-attention sentences—i.e., critical reasoning traces—\sys lowers the probability of hallucination or misleading content and grounds responses in essential reasoning. However, \sys may still occasionally prune low-attention but important reasoning steps, which explains why its accuracy is not always the best across all baselines. 

\subsection{Supplementary Experiments}
In this section, we conduct additional experiments to examine the influence of our method’s performance at different training stages and under different attention functions.

\begin{table}[htp]
    \centering
     \caption{\textbf{Performance of Different Activation Functions.}  
We evaluate the impact of activation functions on method performance under the same training setup in \sys, using Phi-4-Reasoning across four mathematical datasets (GSM8K, MATH500, AIME24, and Minerva).
Pass@1 and token usage (\#Tk) are reported as evaluation metrics, with variance omitted since it is consistently $\leq$ 2\%.  
Best results are shown in \textbf{bold}, and second-best are \underline{underlined}. 
In most settings, \sys achieves the best performance, with $\Entmax$ consistently ranking second.}

     \resizebox{\textwidth}{!}{
    \begin{tabular}{ccccccccccc}
    \toprule
         \multirow{2}{*}{\textbf{Method}}  & \multicolumn{2}{c}{\textbf{GSM8K}} & \multicolumn{2}{c}{\textbf{MATH500}} & \multicolumn{2}{c}{\textbf{AIME24}} & \multicolumn{2}{c}{\textbf{Minerva}} & \multirow{2}{*}{$\overline{\text{Pass@1}}$} & \multirow{2}{*}{$\overline{\#\text{Tk}}$} \\
        \cmidrule(lr){2-3}  \cmidrule(lr){4-5}  \cmidrule(lr){6-7}  \cmidrule(lr){8-9}
        & Pass@1 & \#Tk & Pass@1 & \#Tk & Pass@1 & \#Tk & Pass@1 & \#Tk \\
       \midrule
   Base & \underline{0.9242} & 1017.70 & 0.5480 & 1721.95 & 0.0667 & 1017.70 & 0.2500 & 1898.86 & 0.4472 & 1414.05 \\
   $\Softmax$  & 0.8317 & 1160.63 & 0.4880 & 1379.52 & 0.1333 & 1909.07 & 0.2390 & 1934.72 & 0.4230 & 1595.99\\
    $\Sparsemax$ & 0.8188 & \underline{160.99} & 0.5120 & 451.59 & \underline{0.1667} & 948.60 & 0.2647 & 580.84 & 0.4406 & 535.26\\
    $\Entmax$ & 0.8984 & 163.75 & \underline{0.5520} & \underline{406.97} & \underline{0.1667} & \textbf{876.63} &  \textbf{0.2831} & \underline{439.48} & \underline{0.4751} & \underline{471.71}\\
   $\Softmax_1$ (\sys) & \textbf{0.9311} & \textbf{154.33} &  \textbf{0.5980} & \textbf{344.37} & \textbf{0.2667} & \underline{899.80} & \underline{0.2716} & \textbf{401.19} &\textbf{0.5169} & \textbf{449.92}\\
  \bottomrule
    \end{tabular}
    }
    \label{tab:softmax}
\end{table}

\paragraph{Efficiency of Different Activation Functions.}  
To evaluate the contribution of $\Softmax_1$ in \sys, we conduct experiments comparing \sys with different activation functions: vanilla $\Softmax$, $\Sparsemax$~\cite{hu2024sparse,martins2016softmax}, and $\Entmax$~\cite{wu2023stanhop,correia2019adaptively}.  
Here, $\Entmax$ is a special case of Tsallis $\alpha$-entmax transformations, which interpolate between softmax and sparsemax.
We evaluate these strategies on four datasets—GSM8K, MATH500, AIME24, and Minerva—using Phi-4-Reasoning. 
As shown in \cref{tab:softmax}, the results demonstrate that \sys achieves the best overall performance in both Pass@1 accuracy and token usage. Specifically, the average accuracy increases by \textbf{15.65\%}, while the number of tokens decreases by \textbf{68.18\%} compared to the base model. 
\sys also surpasses the overall performance of $\Sparsemax$ and $\Entmax$, which tend to sharpen both low- and high-attention sentences, potentially cutting off critical reasoning traces. In contrast, \sys is less prone to this issue.
The only exception is the Minerva dataset, where $\Entmax$ attains higher accuracy than \sys while maintaining a similar number of tokens. The underlying reason is difficult to explain at this stage, but it is a pleasant surprise that, except for GSM8K, the overall performance of $\Sparsemax$ and $\Entmax$ does not decline significantly and in some cases even surpasses the base model. This offers a perspective contrary to that of \citet{yang2025less,wang2024reducing}.

\begin{table}[htp]
    \centering
    \vspace{-0.5em}
     \caption{\textbf{Outlier Removal Performance in FROST.} 
We evaluate outlier removal performance on the AIME2024 dataset using the Phi-4-Reasoning model. 
As outlier metrics, we report the maximum infinity norm $\norm{\mathbf{x}}_{\infty}$  and average kurtosis of the activation tensors. 
To assess the proportion of critical traces, we also report the average sentence entropy before and after applying \sys. 
All results are reported with variance omitted, as it is consistently $\leq$ 2\%. 
Best results are shown in \textbf{bold}, and second-best results are \underline{underlined}. In most settings, \sys achieves the best performance in outlier removal and yields higher average sentence entropy. These metrics demonstrate that our method effectively removes reasoning outliers, thereby improving both reasoning performance and efficiency.
}
 
    \resizebox{\textwidth}{!}{
   \begin{tabular}{lccccc}
    \toprule
         \textbf{Method} & Maximum Infinity Norm $\norm{\mathbf{x}}_{\infty} \downarrow$ & Average Kurtosis $\downarrow$ & Average Sentence Entropy $\uparrow$ & Pass@1 $\uparrow$ & \#Tk $\downarrow$ \\
       \midrule
   Base  & 35.31 & 241.72 & 2.71 & 0.0667 & 1017.70 \\
   $\Softmax$ & 34.53 & 189.36 & 2.79 &  0.1333 & 1909.07\\
   $\Sparsemax$ & 34.06 &  152.18 & \underline{2.93} & \underline{0.1667} & 948.60\\
   $\Entmax$ & \underline{30.39} & \underline{43.72} & 2.92 & \underline{0.1667} & \textbf{876.63}\\
   \sys &  \textbf{29.67} & \textbf{21.54} & \textbf{3.07} & \textbf{0.2667} & \underline{899.80} \\  \bottomrule
\end{tabular}
    }
    \label{tab:outlier}
\end{table}

\paragraph{Outlier Removal Performance in FROST.}  
To evaluate the performance of \sys in removing attention outliers, we employ two outlier-specific metrics: the \textit{maximum infinity norm} $\norm{\mathbf{x}}_{\infty}$ of the activation tensors $\mathbf{x}$ across all Transformer layers, and the \textit{average kurtosis} of $\mathbf{x}$, which together quantify the presence of outliers.
In addition, to demonstrate that removing attention outliers increases the probability assigned to critical sentences, we introduce an entropy-based evaluation metric. 
Following \citet{wang2025beyond}, token entropy serves as an indicator of criticality: critical tokens tend to exhibit higher entropy than non-critical ones.  
When a sentence contains more critical tokens, it is expected to exert a stronger influence on final answer generation.  
Accordingly, we use average sentence entropy to assess whether the reasoning traces in \sys become more critical after training.  
In our experiments, we analyze these metrics on the AIME2024 dataset using the Phi-4-Reasoning model and compare them with the base model. As shown in \cref{tab:outlier}, \sys effectively reduces outliers, evidenced by lower maximum infinity norm  $\norm{\mathbf{x}}_{\infty}$  and average kurtosis values. Furthermore, the increase in average sentence entropy indicates that \sys strengthens the model’s focus on critical reasoning traces, thereby improving reasoning efficiency. 
Specifically, we reduce the maximum infinity norm $\norm{\mathbf{x}}_{\infty}$ by \textbf{15.97\%} and the average kurtosis by \textbf{91.09\%.} In addition, the average sentence entropy increases by \textbf{13.28\% }compared to the base model.
Additionally, the results show that reasoning outlier metrics—maximum infinity norm $\norm{\mathbf{x}}_{\infty}$ and average kurtosis—are closely related to model performance and average sentence entropy. Higher outlier values correspond to lower sentence entropy and less efficient reasoning traces. This further supports that the reasoning-outlier removal contributes to more efficient reasoning. The only exception is that the average sentence entropy of $\Sparsemax$ is similar to $\Entmax$, while the reasoning outlier values of $\Entmax$ are much smaller than those of $\Sparsemax$. 
\revise{A plausible explanation is that both $\Entmax$ and $\Sparsemax$ act as sharpening activations that jointly suppress low- and high-valued attention scores. This bidirectional truncation can inadvertently remove parts of crucial reasoning traces, lowering average sentence entropy and reducing Pass@1 performance. Meanwhile, attention outlier metrics such as the maximum infinity norm and kurtosis primarily reflect internal activation dynamics rather than output quality, explaining their relative stability despite external performance declines. Since both activations reshape attention distributions similarly, their outputs also appear alike—with comparable Pass@1 and entropy values—though $\Entmax$’s smoother contraction yields slightly less degradation in outlier metrics. Overall, this indicates that excessive sharpening can eliminate valuable reasoning signals even while suppressing attention outliers, highlighting $\Softmax_1$’s advantage through selective tail contraction.
}

\subsection{Generalizability of Model}
In this section, we evaluate the generalization ability of \sys on out-of-domain reasoning tasks to verify that its improvements do not harm, but rather preserve or enhance, the model’s generation quality beyond the training domain. Using Phi-4-Reasoning as the base model, we test on three additional reasoning benchmarks—LeetCode \cite{xia2025leetcodedataset}, LiveCodeBench \cite{jain2024livecodebench}, and UGPhysical \cite{xu2025ugphysics}—covering both coding and physical reasoning tasks. The results in \cref{tab:genal} show that \sys preserves—and even improves—generalization to unseen reasoning tasks. This is expected because \sys filters out uncritical reasoning traces in a manner that generalizes beyond the specific tasks used during fine-tuning. Since \sys only replaces the attention activation with $\Softmax_1$ and uses lightweight LoRA updates, the parameter shift is minimal, ensuring that the model’s broader reasoning ability remains intact.

\begin{table}[ht]
    \centering
    \caption{\revise{\textbf{FROST Generalization on Other Reasoning Tasks.} We evaluate the generalization of FROST using Phi-4-Reasoning across three out-of-domain reasoning tasks (Code and Physics). Pass@1 accuracy and token usage (\#Tk) are reported, with variance consistently $\leq$ 2\%. Best results are in bold, and second-best are underlined. FROST consistently achieves top performance, demonstrating strong generalization across reasoning domains.}}
    \label{tab:genal}
    
    \resizebox{\textwidth}{!}{
    \begin{tabular}{ccccccccc}
    \toprule
         \multirow{2}{*}{\textbf{Method}}  & \multicolumn{2}{c}{\textbf{Leetcode}} & \multicolumn{2}{c}{\textbf{LiveCodeBench}} & \multicolumn{2}{c}{\textbf{UGPhysics}} & \multirow{2}{*}{$\overline{\text{Pass@1}}$} & \multirow{2}{*}{$\overline{\#\text{Tk}}$} \\
        \cmidrule(lr){2-3}  \cmidrule(lr){4-5}    \cmidrule(lr){6-7}
        & Pass@1 & \#Tk & Pass@1 & \#Tk & Pass@1 & \#Tk  \\
       \midrule
   Base & 0.3222 & 2755.13 & 0.3248 & 3154.80 &  \underline{0.3172} & \underline{2603.00} & 0.3214 & 2837.64 \\
   $\Softmax$  & \underline{0.3778} & \underline{2106.85}  & \underline{0.3538} & \underline{2909.07} & 0.3011 & 2622.52 & \underline{0.3442} & \underline{2546.15}\\
   \sys & \textbf{0.3889} & \textbf{1163.06} &   \textbf{0.3777} & \textbf{1967.56} & \textbf{0.3473} & \textbf{805.77} & \textbf{0.3713} & \textbf{1312.13}\\
  \bottomrule
    \end{tabular}
    }
    \label{tab:genal}
\end{table}

%% file: 7conclusion.tex
We propose an attention-aware efficient reasoning method, \textbf{FROST}. Our approach introduces the concept of \textit{reasoning outliers}, which contribute to high latency in reasoning performance, and provides an outlier removal mechanism that enables LRMs to leverage lightweight supervised fine-tuning (SFT) for generating efficient and accurate reasoning traces in mathematical problem solving. Theoretically, we show that the existing outlier removal technique $\Softmax_1$ is effective for reasoning outliers and can operate at the sentence level. Empirically, FROST improves response accuracy by \revise{\textbf{26.70\%} and reduces token usage by \textbf{69.68\%}}. In addition, it decreases the maximum infinity norm $\norm{\mathbf{x}}_{\infty}$  by \textbf{15.97\%} and average kurtosis by \textbf{91.09\%}, confirming its effectiveness in mitigating reasoning outliers.

Although \sys achieves strong performance in efficient reasoning, several limitations remain. First, our method is currently restricted to mathematical reasoning tasks, while many reasoning models also target domains such as coding. Second, \sys relies solely on supervised fine-tuning and does not incorporate GRPO, which could further enhance efficiency. In future work, we plan to extend \sys to additional reasoning tasks, including coding, and to develop a GRPO-based approach that builds on our current findings to further improve efficient reasoning performance.

%% file: x_acknowledgments.tex
HL would like to thank Dennis Wu for insightful discussions on theoretical analysis.
This research was supported in part by the computational resources and staff contributions of the RTRC High Performance Computing Cluster. 
Additional support was provided by the Quest High Performance Computing Facility at Northwestern University, which is jointly supported by the Office of the Provost, the Office for Research, and Northwestern University Information Technology.
The content is solely the responsibility of the authors and does not necessarily represent the official
views of the funding agencies.

%% file: disclouse.tex
\section*{Ethical Statement}
This work investigates reasoning outliers in large reasoning models (LRMs) and proposes an outlier-removal technique to mitigate attention outliers. In line with the ICLR Code of Ethics\footnote{\url{https://iclr.cc/public/CodeOfEthics}}
, we acknowledge that our method may inadvertently amplify biases present in training data, potentially leading to unfair outcomes for underrepresented groups. Prior studies have also noted that supervised fine-tuning (SFT) can induce shallow alignment and affect red-teaming protection, but this issue is outside the scope of our work. Our focus is on improving reasoning efficiency, and we believe this research does not raise serious ethical concerns.

\section*{Reproducibility}
\label{sec:reproduce}
To ensure reproducibility, we release an anonymous open-source repository (\href{https://anonymous.4open.science/r/GAPO-DADB}{link}
) containing the full implementation of \sys and selected baselines, with plans for full open-sourcing upon acceptance. All experiments are conducted with three random seeds, yielding stable results with standard deviations below 2\%. We adopt a unified training setup using the AdamW optimizer with learning rate $1\times10^{-5}$, batch size 8, and batch size 256 for deployment. For low-rank adaptation, we set the LoRA rank to 8 and LoRA $\alpha$ to 16. Detailed hyperparameters are provided in \cref{app:hyper}, and theoretical proofs are included in \cref{app:prove}.

%% file: appendix.tex
\part*{Supplementary Material}

{
\setlength{\parskip}{-0em}
\startcontents[sections]
\printcontents[sections]{ }{1}{}
}

\section{Impact Statement}
\label{ap:impact}
\input{boarder}

\section{Additional Related Work}
\textbf{Efficient Alignment.}
In recent years, foundation models \cite{zhou2025genomeocean,he2025rhythm,he2024st,wang2024qwen2,touvron2023llama} have shown strong capabilities in solving multitask problems. To further improve their performance on specific tasks, alignment techniques are essential for refining model behavior. However, traditional approaches like RLHF \cite{ouyang2022training} and DPO \cite{rafailov2023direct} are computationally expensive. 
This highlights the urgent need for parameter-efficient fine-tuning methods that offer effective and economical alignment for foundation models. Several traditional methods demonstrate strong capabilities in aligning foundation models, including LoRA \cite{hu2021lora} and QLoRA \cite{dettmers2024qlora}. Building on this, \citet{luo2025fast} propose a LoRA variant that replaces the standard softmax layer with OutEffHop layers \cite{hu2024outlier} to improve the efficiency of low-rank adaptation. 
However, all of these methods are heavily based on LoRA, and when adaptation is required for modules outside the attention architecture, the computational cost increases significantly. 
\citet{zhao2025understanding, luo2024decoupled} propose novel alignment methods that focus on small subsets of neurons within foundation models. For example, \citet{zhao2025understanding} identify key neurons with high influence on LLMs' jailbreak defense using latent representations, and fine-tune only these neurons using red-teaming datasets. 
Our method builds on fast low-rank adaptation techniques \cite{luo2025fast}, further improving adaptation efficiency, and integrates them into SFT training to optimize reasoning paths and produce efficient reasoning models.

\section{Proofs of Main Text}
\label{app:prove}
\subsection{\cref{lem:pool-order}}
\label{proof:pool-order}
\textit{Proof of \cref{lem:pool-order}.}
Monotonicity means that if we increase any input coordinate to $\phi$, its output does not decrease.
Let $u=\{z_t\}_{t\in S_i}$ and $w=\{z_{t'}\}_{t'\in S_j}$.
If for each coordinate of $u$ there is a not-smaller coordinate in $w$ replaced by the smaller value, then by repeatedly applying coordinatewise monotonicity we obtain
$
s_i=\phi(u) \ \ge\ \phi(w)=s_j.
$
Order preservation (P1) then yields $\alpha_i\ge \alpha_j$.

\subsection{\cref{thm:sentence-outlier}}
\label{proof:sentence-outlier}
\textit{Proof of \cref{thm:sentence-outlier}.} By Lemma~\ref{lem:pool-order}, sentence scores $s$ reflect dominance induced by token compatibilities under $\phi$.
Applying Assumption~\ref{ass:softmax1}(P3) directly to $s$ yields \eqref{eq:tail-contract}.
Assumption~\ref{ass:softmax1}(P2) allows re-centering $s\leftarrow s-c\mathbf{1}$ without changing $\alpha$; thus \eqref{eq:tail-contract} is invariant to any global shift and depends only on relative separations.

\subsection{\cref{thm:deploy}}
\label{proof:deploy}
\textit{Proof of \cref{thm:deploy}.} For \eqref{eq:logit-one-layer}, apply operator-norm submultiplicativity:
$\|W_o(\alpha_i v_i)\|\le \|W_o\|_{\mathrm{op}}\cdot \alpha_i\|v_i\|\le B_o\,\varepsilon\,B_v$.
To obtain \eqref{eq:logit-L-layers}, propagate the perturbation through $L$ differentiable layers with Jacobians $J_\ell$:
\begin{equation*}
\|\Delta \ell_i^{(L)}\|
\ \le\
\Bigl(\prod_{\ell=1}^L \|J_\ell\|_{\mathrm{op}}\Bigr)\,
\|W_o\|_{\mathrm{op}}\,
\alpha_i \|v_i\|
\ \le\
\varepsilon
\Bigl(\prod_{\ell=1}^L B_\ell\Bigr)
B_o B_v.
\end{equation*}
Finally, since $\Softmax_1$ is $1$-Lipschitz in $\ell_\infty\!\to\!\ell_1$, the change in probabilities is bounded by the logit change, yielding \eqref{eq:prob-shift}.
Replacing $\prod_{\ell=1}^L B_\ell$ with $B^L$ (by definition of $B$) gives the stated $O(B_o B_v B^L \varepsilon)$ rate.
If $B_o,B_v,B$ are $O(1)$, the rate simplifies to $O(\varepsilon)$.

\revise{
\section{An Example of LRM Reasoning Traces}
In this section, we analyze the Phi-4-Reasoning response to the first question of AIME24, which is also illustrated in \cref{fig:theoritial}. As shown in the color box in \cref{fig:lrm_trace}, traces S1 and S2 are classified as uncritical. Although S2 includes partially critical content such as \texttt{"So the walking time (actual walking time) plus t minutes equals total time."}, its overall reasoning remains non-critical. Trace S3 represents a critical reasoning step, where the model identifies the two key equations in the problem. Subsequently, from S4 to S19, the model enters a self-verification phase, producing reasoning traces beginning with \texttt{wait} that reflect self-checking and correction. Starting from S20, the model resumes critical reasoning after the signal \texttt{"We’ll produce final answer in a box."}, and by S24, it generates the final answer, concluding its reasoning process.

\begin{tcolorbox}[colback=black!5!white,colframe=black,title=\textbf{An Example of LRM Reasoning Traces},floatplacement=t,float,label={fig:lrm_trace}]
\textcolor{darkgreen}{\textbf{[Question]}} Every morning Aya walks $9$ km and then spends $t$ minutes in a coffee shop. 
At speed $s$ km/h, the total time is $4$ hours (including the coffee stop). 
At speed $s+2$ km/h, the total time is $2$ hours $24$ minutes. 
If she walks at $s+\tfrac{1}{2}$ km/h, find the total time in minutes (including the coffee stop).

\textcolor{red}{\textbf{[Uncritical Trace S1]}} <think>We are given: "Every morning Aya goes for a 9-kilometer-long walk and stops at a coffee shop afterwards. When she walks at a constant speed of $s$ kilometers per hour, the walk takes her 4 hours, including $t$ minutes spent in the coffee shop. When she walks $s+2$ kilometers per hour, the walk takes her 2 hours and 24 minutes, including \$t\$ minutes spent in the coffee shop. Suppose Aya walks at $s+\frac{1}{2}$ kilometers per hour. Find the number of minutes the walk takes her, including the \$t\$ minutes spent in the coffee shop."

\textcolor{red}{\textbf{[Uncritical Trace S2]}} Let's denote: "Walking distance = 9 km. There's a coffee shop stop time $t$ minutes, but note that the given total times includes the coffee shop stop time. So the walking time (actual walking time) plus $t$ minutes equals total time.
Let $s$ be her walking speed (km/h) in first scenario. Then walking time (in hours) is $9/s$. And the total time is 4 hours."

\textcolor{blue}{\textbf{[Critical Trace S3]}}  But includes $t$ minutes stop, so actual walking time is 4 hours minus $t$ minutes (converted to hours, $t/60$ hours), so equation: $9/s = 4 - t/60$. (Equation 1) and $9/(s+2) = 2.4 - t/60$ (Equation 2)

\textcolor{red}{\textbf{[Uncritical Trace S4]}} But wait, let's check: speed = $s+2$, then walking time (in hours) is $9/(s+2)$. Total time is 2 hours 24 minutes, which is $2 + 24/60 = 2.4$ hours.

\textcolor{red}{\textbf{[Uncritical Trace S5]}} But wait, let's check: The problem said: "Suppose Aya walks at s+1/2 kilometers per hour. Find the number of minutes the walk takes her, including the t minutes spent in the coffee shop." 

\textcolor{red}{\textbf{[Uncritical Trace S6]}} Wait, 2 hours 24 minutes = 144 minutes. 

\textcolor{red}{\textbf{[Uncritical Trace S7]}} Let me recalc: 3 hours 24 minutes = 2*60 + 24 = 120 + 24 = 144 minutes.

\textcolor{red}{\textbf{[Uncritical Trace S8]}} Let us double-check the reasoning process.

\textcolor{red}{\textbf{[Uncritical Trace S9-S17]}} \red{.....}

\textcolor{red}{\textbf{[Uncritical Trace S18]}} I'll now produce step-by-step explanation. 

\textcolor{red}{\textbf{[Uncritical Trace S19]}} We'll produce final answer in a box.

\textcolor{blue}{\textbf{[Critical Trace S20]}}  The two equations are given as $9/s = 4 - t/60$ (Equation 1) and $9/(s+2) = 2.4 - t/60$ (Equation 2).

\textcolor{blue}{\textbf{[Critical Trace S21]}}  Alternatively, we can solve for $t$ in terms of $s$ from equation 1: $t = 60*(4 - 9/s) = 240 - 540/s$, and plug in equation 2: $9/(s+2) = 2.4 - (240 - 540/s)/60 = 2.4 - (4 - 9/s) = 2.4 - 4 + 9/s = (9/s) - 1.6$. 

\textcolor{blue}{\textbf{[Critical Trace S22]}} So equation becomes: $9/(s+2) = (9/s) - 1.6$. Multiply both sides by $s(s+2)$: $9s = (9(s+2)) - 1.6 s(s+2)$.

\textcolor{blue}{\textbf{[Critical Trace S23]}} Now, when Aya walks at speed $s+1/2 = 2.5 + 0.5 = 3$ km/h, then walking time (in hours) is $9/3 = 3$ hours. But total time includes coffee shop stop time $t$ minutes, which is 24 minutes = $24/60 = 0.4$ hours. So total time is $3 + 0.4 = 3.4$ hours, which is 3 hours 24 minutes.

\textcolor{blue}{\textbf{[Critical Trace S24]}} I'll produce final answer: \boxed{204}. </think>

\end{tcolorbox}
}
\section{
Experiment System and Implement Settings}

\subsection{Computational Resources}
\label{app:resource}
We perform all experiments using two NVIDIA H100 GPUs with 80GB of memory and a 12-core INTEL(R) XEON(R) PLATINUM 8592 CPU operating at 1.90GHz. Our code is developed in PyTorch and utilizes the Hugging Face Transformer Library for experimental execution. For running the LLMs, we use the default system prompt provided by the official source and set the temperature to 0.6 to balance consistency and performance.

\subsection{Hyperparameters}
\label{app:hyper}
We present the hyperparameters used in the fine-tuning stage for each model. 
We use \textbf{AdamW} \cite{loshchilov2017decoupled} as the optimizer. 
Most other hyperparameters are kept consistent across all models and datasets, including a batch size of 256 during deployment and 8 during training. In training, we also use gradient accumulation with 4 steps and set the weight decay to 0.01 for all training runs.
A learning rate of $1e^{-5}$ is used for all models during fine-tuning. 
For low-rank adaptation, we use a LoRA rank of 8 and LoRA alpha set to 16. In \sys, we set the maximum training steps to 5,000.
All supervised fine-tuning and GRPO training are conducted using mixed precision with bfloat16.
In deployment, we set the temperature to 0.6 for all models with top-$p$ sampling at 0.9. For evaluation, we use a maximum generation length of 4096 across all models, except TALE.

\section{Additional Experiments}
In this section, we present additional experiments demonstrating that \sys surpasses current state-of-the-art efficient reasoning methods.

\subsection{Training and Test Time Comparison}
We conduct experiments to measure the training and inference time of each baseline and compare their computational costs with \sys. 
For evaluation, test time is measured on the AIME tasks with the GPT-OSS-20B model, while training time is reported using the respective datasets specified in each baseline’s original paper. 
All experiments are conducted on the same computational resources, as described in \cref{app:resource}.  

\begin{table}[ht]
    \centering
    \caption{\textbf{Comparison of Training and Test Time Costs Across Methods.} We conduct experiments to measure the training and test time of each method. For test-time evaluation, we use the AIME dataset with the GPT-OSS-20B model. Best results are shown in \textbf{bold}, and second-best are \underline{underlined}.  }
    \vspace{-0.1in}
    \begin{tabular}{lcccc}
    \toprule
      Method   &  TALE & DRP & ThinkLess & \sys\\
      \midrule
     Training Time (m)   & -  & \underline{353} & 1186 &  \textbf{204} \\
     Test Time (m) &  56 & 18.5 & \underline{4.2} & \textbf{3}   \\
     \bottomrule
    \end{tabular}
    
    \label{tab:time}
\end{table}

As shown in \cref{tab:time}, \sys achieves the fastest training time among all methods, while also minimizing computation cost and inference time during deployment. This demonstrates that our approach not only accelerates training but also reduces deployment overhead.

\subsection{Attention Distributions of Activation Functions}
We conduct an additional experiment to analyze the attention distribution of GPT-OSS-20B on a sample from the GSM8K dataset. As shown in \cref{fig:distribution}, \sys effectively removes a large number of low-attention sentences while retaining significant ones. In contrast, the vanilla model produces many sentences with low attention weights, and $\Sparsemax$ and $\Entmax$ retain only one to two sentences, often aggressively discarding important reasoning traces. This visualization provides an explanation consistent with the performance results reported in \cref{tab:softmax}.

\begin{figure}[htp]
    \centering
    \includegraphics[width=\linewidth]{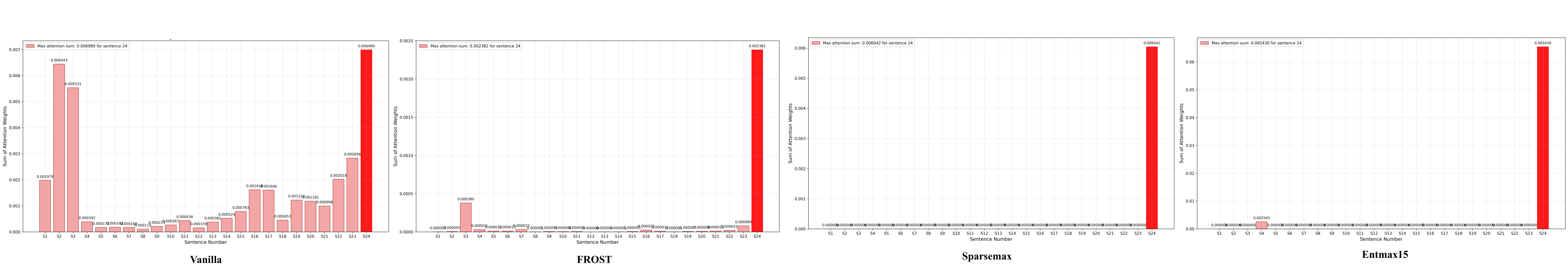}
    \caption{\textbf{Attention Distribution of Each Activation Function.}}
    \label{fig:distribution}
\end{figure}

\revise{
\section{Influence the attention dynamics of $\Softmax_1$ during training and inference }
We observe that incorporating $\Softmax_1$ significantly influences both training and inference attention dynamics across transformer layers. During supervised fine-tuning (SFT), $\Softmax_1$ enforces tail contraction by suppressing low-attention activations, which stabilizes gradients and reduces the variance of updates propagated through residual connections. This effect leads to faster convergence of LoRA adapters, as the low-rank parameter subspace more efficiently aligns with critical attention directions, improving overall adaptation coverage within fewer training steps. This observation is consistent with \citet{luo2025fast,hu2024outlier}. Across layers, $\Softmax_1$ reshapes the attention landscape—shallow layers become more selective in contextual grounding, while deeper layers exhibit higher entropy concentration around critical reasoning traces. During inference, this sharpening propagates forward, effectively filtering redundant reasoning sentences while maintaining coherence. Together, these behaviors demonstrate that $\Softmax_1$ not only enhances efficient reasoning but also accelerates LoRA-SFT optimization by improving the representational focus of each attention head.

\section{Influence of $\Softmax_1$ Across Layers}
We analyze the effect of $\Softmax_1$ across transformer layers by visualizing the attention distributions of head 15 for both vanilla $\Softmax$ and $\Softmax_1$. As shown in \cref{fig:all_layer,fig:all_layer_softmax1}, $\Softmax_1$ consistently suppresses attention outliers, leading to smoother and more stable activations across the network. In lower layers, $\Softmax_1$ contracts heavy tails and mitigates rare extreme peaks, enhancing local feature mixing with higher-entropy and reduced kurtosis distributions. In higher layers, it suppresses residual long-range spikes and sharpens focus on semantically relevant tokens, yielding sparser yet more stable attention and clearer causal information flow.

\begin{figure}[h]
    \centering
    \includegraphics[width=\linewidth]{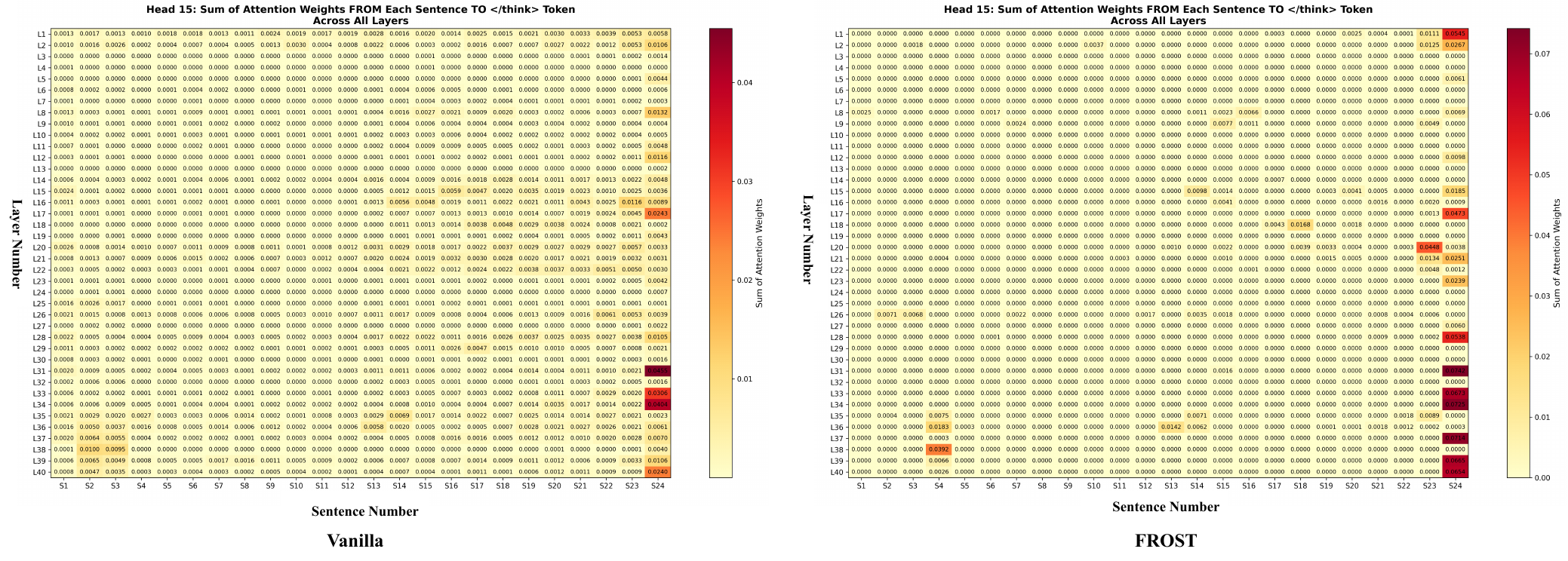}
    \caption{\textbf{Theoretical Analysis of Reasoning Outlier Removal in All Layers}}
    \label{fig:all_layer}
\end{figure}
\begin{figure}[ht]
    \centering
    \includegraphics[width=\linewidth]{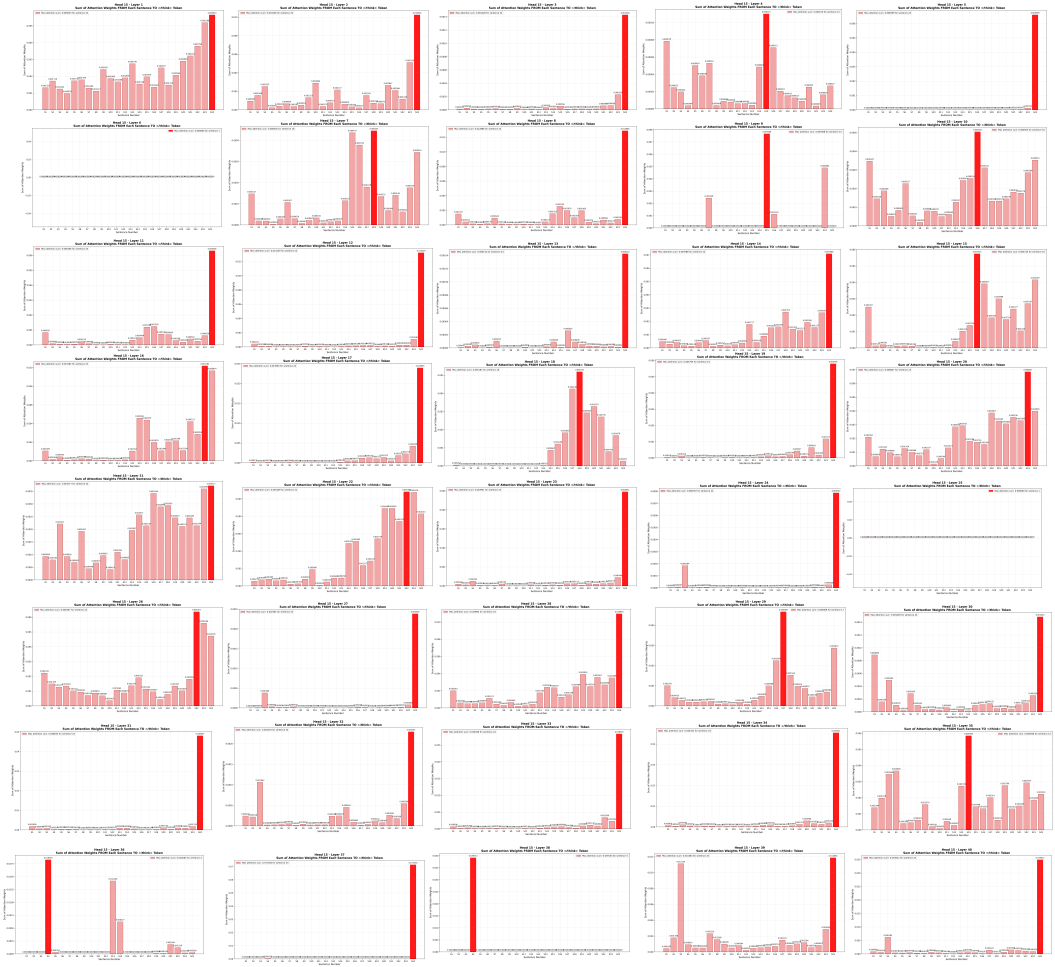}
    \caption{\textbf{Attention Distribution of $\Softmax_1$ Across All Layers}}
    \label{fig:all_layer_softmax1}
\end{figure}

\section{Extended Attention Heatmaps Across Additional Layers and Heads}
In this section, we present extended attention heatmaps covering additional layers and heads. Specifically, we analyze \textbf{Layers 0, 5, 15, 25, 30, 35, and 39} and \textbf{Heads 0, 5, 10, 15, 20, 25, 30, 35, and 39} to provide a more comprehensive view of attention evolution across the network. The corresponding observations are illustrated in \cref{fig:all_layer_heatmap}. 
\begin{figure}[!ht]
    \centering
    \includegraphics[width=0.45\linewidth]{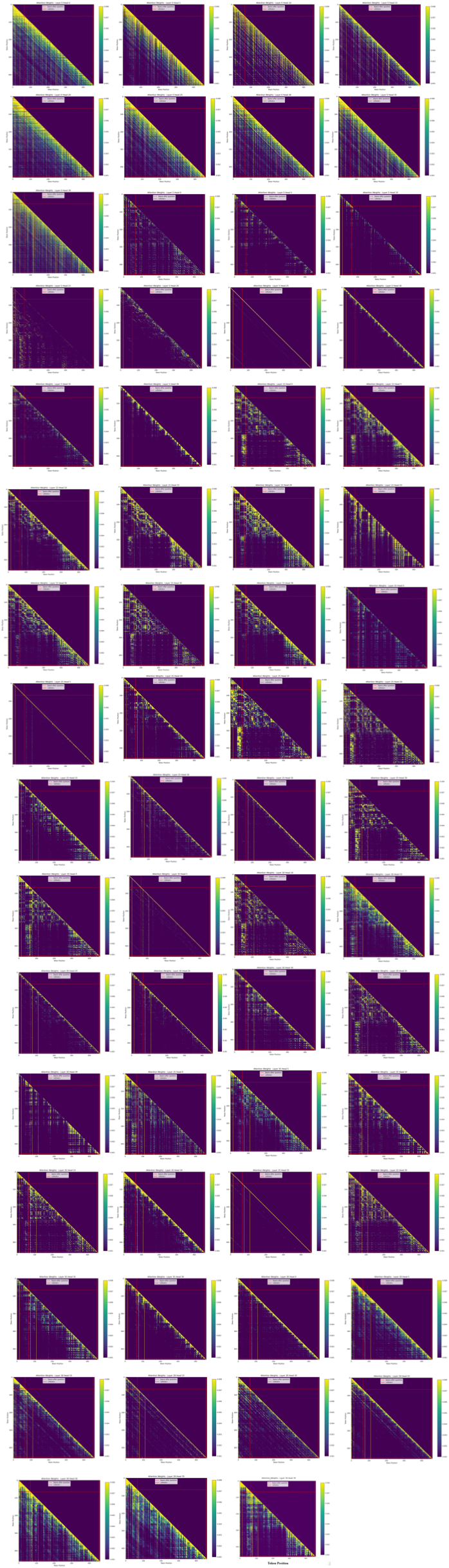}
    \caption{\textbf{Extended Attention Heatmaps Across Additional Layers and Heads}}
    \label{fig:all_layer_heatmap}
\end{figure}

\section{Human Expert Evaluation}
We invite three computer science students specializing in reasoning models to annotate reasoning traces generated by the original and \sys-trained models. We then compare the traces pruned by \sys and evaluate their criticality based on relevance and contribution to the final answer. Averaging across all evaluators, \sys achieves \textbf{92\%} accuracy in correctly removing non-critical reasoning traces.
Only \textbf{8\%} of reasoning traces are incorrectly removed, which significantly degrades final-answer accuracy. These mistakenly pruned traces are typically long and contain repeated information that supports self-verification and error correction. However, they also provide critical content—such as key equations—in the end of trace. This observation suggests a potential explanation for why \sys achieves the second-best Pass@1 score in the Phi-4-Reasoning experiment shown in \cref{tab:main}.
}

\section{Disclosure of LLM Usage}
\label{ap:llm}
In our paper and project, we use large language models (LLMs) to help revise the text for greater conciseness and precision.

%% file: boarder.tex
We believe this methodology offers an opportunity to strengthen the core of large reasoning models by improving efficiency and enabling models to produce more critical reasoning traces. However, it may also amplify biases present in the training data, potentially leading to unfair or discriminatory outcomes for underrepresented groups.

%% file: main_iclr26.bbl
\begin{thebibliography}{74}
\providecommand{\natexlab}[1]{#1}
\providecommand{\url}[1]{\texttt{#1}}
\expandafter\ifx\csname urlstyle\endcsname\relax
  \providecommand{\doi}[1]{doi: #1}\else
  \providecommand{\doi}{doi: \begingroup \urlstyle{rm}\Url}\fi

\bibitem[Abdin et~al.(2025)Abdin, Agarwal, Awadallah, Balachandran, Behl, Chen, de~Rosa, Gunasekar, Javaheripi, Joshi, et~al.]{abdin2025phi}
Marah Abdin, Sahaj Agarwal, Ahmed Awadallah, Vidhisha Balachandran, Harkirat Behl, Lingjiao Chen, Gustavo de~Rosa, Suriya Gunasekar, Mojan Javaheripi, Neel Joshi, et~al.
\newblock Phi-4-reasoning technical report.
\newblock \emph{arXiv preprint arXiv:2504.21318}, 2025.

\bibitem[Agarwal et~al.(2025)Agarwal, Ahmad, Ai, Altman, Applebaum, Arbus, Arora, Bai, Baker, Bao, et~al.]{agarwal2025gpt}
Sandhini Agarwal, Lama Ahmad, Jason Ai, Sam Altman, Andy Applebaum, Edwin Arbus, Rahul~K Arora, Yu~Bai, Bowen Baker, Haiming Bao, et~al.
\newblock gpt-oss-120b \& gpt-oss-20b model card.
\newblock \emph{arXiv preprint arXiv:2508.10925}, 2025.

\bibitem[Besta et~al.(2024)Besta, Blach, Kubicek, Gerstenberger, Podstawski, Gianinazzi, Gajda, Lehmann, Niewiadomski, Nyczyk, et~al.]{besta2024graph}
Maciej Besta, Nils Blach, Ales Kubicek, Robert Gerstenberger, Michal Podstawski, Lukas Gianinazzi, Joanna Gajda, Tomasz Lehmann, Hubert Niewiadomski, Piotr Nyczyk, et~al.
\newblock Graph of thoughts: Solving elaborate problems with large language models.
\newblock In \emph{Proceedings of the AAAI Conference on Artificial Intelligence}, volume~38, pp.\  17682--17690, 2024.

\bibitem[Brown et~al.(2020)Brown, Mann, Ryder, Subbiah, Kaplan, Dhariwal, Neelakantan, Shyam, Sastry, Askell, et~al.]{brown2020language}
Tom Brown, Benjamin Mann, Nick Ryder, Melanie Subbiah, Jared~D Kaplan, Prafulla Dhariwal, Arvind Neelakantan, Pranav Shyam, Girish Sastry, Amanda Askell, et~al.
\newblock Language models are few-shot learners.
\newblock \emph{Advances in neural information processing systems}, 33:\penalty0 1877--1901, 2020.

\bibitem[Cai et~al.()Cai, Xiao, Sun, Luo, Zhang, Wan, Li, Zhou, Chang, Gu, et~al.]{citation-0}
Zefan Cai, Wen Xiao, Hanshi Sun, Cheng Luo, Yikai Zhang, Ke~Wan, Yucheng Li, Yeyang Zhou, Li-Wen Chang, Jiuxiang Gu, et~al.
\newblock R-kv: Redundancy-aware kv cache compression for reasoning models.
\newblock In \emph{The Thirty-ninth Annual Conference on Neural Information Processing Systems}.

\bibitem[Chen et~al.(2025)Chen, Xu, Liang, He, Pang, Yu, Song, Liu, Zhou, Zhang, Wang, Tu, Mi, and Yu]{chen2024not}
Xingyu Chen, Jiahao Xu, Tian Liang, Zhiwei He, Jianhui Pang, Dian Yu, Linfeng Song, Qiuzhi Liu, Mengfei Zhou, Zhuosheng Zhang, Rui Wang, Zhaopeng Tu, Haitao Mi, and Dong Yu.
\newblock Do {NOT} think that much for 2+3=? on the overthinking of long reasoning models.
\newblock In \emph{Forty-second International Conference on Machine Learning}, 2025.

\bibitem[Chia et~al.(2024)Chia, Chen, Xu, Luu, Poria, and Bing]{chia-etal-2024-reasoning}
Yew~Ken Chia, Guizhen Chen, Weiwen Xu, Anh~Tuan Luu, Soujanya Poria, and Lidong Bing.
\newblock Reasoning paths optimization: Learning to reason and explore from diverse paths.
\newblock In Yaser Al-Onaizan, Mohit Bansal, and Yun-Nung Chen (eds.), \emph{Findings of the Association for Computational Linguistics: EMNLP 2024}, Miami, Florida, USA, November 2024. Association for Computational Linguistics.

\bibitem[Choi et~al.(2025)Choi, Lee, Tack, Song, Dingliwal, Jayanthi, Ganesh, Shin, Galstyan, and Bodapati]{choi2025think}
Daewon Choi, Jimin Lee, Jihoon Tack, Woomin Song, Saket Dingliwal, Sai~Muralidhar Jayanthi, Bhavana Ganesh, Jinwoo Shin, Aram Galstyan, and Sravan~Babu Bodapati.
\newblock Think clearly: Improving reasoning via redundant token pruning.
\newblock In Christos Christodoulopoulos, Tanmoy Chakraborty, Carolyn Rose, and Violet Peng (eds.), \emph{Findings of the Association for Computational Linguistics: EMNLP 2025}, pp.\  21437--21451, Suzhou, China, November 2025. Association for Computational Linguistics.
\newblock ISBN 979-8-89176-335-7.

\bibitem[Cobbe et~al.(2021)Cobbe, Kosaraju, Bavarian, Chen, Jun, Kaiser, Plappert, Tworek, Hilton, Nakano, et~al.]{cobbe2021training}
Karl Cobbe, Vineet Kosaraju, Mohammad Bavarian, Mark Chen, Heewoo Jun, Lukasz Kaiser, Matthias Plappert, Jerry Tworek, Jacob Hilton, Reiichiro Nakano, et~al.
\newblock Training verifiers to solve math word problems.
\newblock \emph{arXiv preprint arXiv:2110.14168}, 2021.

\bibitem[Comanici et~al.(2025)Comanici, Bieber, Schaekermann, Pasupat, Sachdeva, Dhillon, Blistein, Ram, Zhang, Rosen, et~al.]{comanici2025gemini}
Gheorghe Comanici, Eric Bieber, Mike Schaekermann, Ice Pasupat, Noveen Sachdeva, Inderjit Dhillon, Marcel Blistein, Ori Ram, Dan Zhang, Evan Rosen, et~al.
\newblock Gemini 2.5: Pushing the frontier with advanced reasoning, multimodality, long context, and next generation agentic capabilities.
\newblock \emph{arXiv preprint arXiv:2507.06261}, 2025.

\bibitem[Correia et~al.(2019)Correia, Niculae, and Martins]{correia2019adaptively}
Gon{\c{c}}alo~M. Correia, Vlad Niculae, and Andr{\'e} F.~T. Martins.
\newblock Adaptively sparse transformers.
\newblock In Kentaro Inui, Jing Jiang, Vincent Ng, and Xiaojun Wan (eds.), \emph{Proceedings of the 2019 Conference on Empirical Methods in Natural Language Processing and the 9th International Joint Conference on Natural Language Processing (EMNLP-IJCNLP)}, pp.\  2174--2184, Hong Kong, China, November 2019. Association for Computational Linguistics.

\bibitem[Cuadron et~al.(2025)Cuadron, Li, Ma, Wang, Wang, Zhuang, Liu, Schroeder, Xia, Mao, et~al.]{cuadron2025danger}
Alejandro Cuadron, Dacheng Li, Wenjie Ma, Xingyao Wang, Yichuan Wang, Siyuan Zhuang, Shu Liu, Luis~Gaspar Schroeder, Tian Xia, Huanzhi Mao, et~al.
\newblock The danger of overthinking: Examining the reasoning-action dilemma in agentic tasks.
\newblock \emph{arXiv preprint arXiv:2502.08235}, 2025.

\bibitem[Dettmers et~al.(2023)Dettmers, Pagnoni, Holtzman, and Zettlemoyer]{dettmers2024qlora}
Tim Dettmers, Artidoro Pagnoni, Ari Holtzman, and Luke Zettlemoyer.
\newblock Qlora: Efficient finetuning of quantized llms.
\newblock In \emph{The Thirty-seventh Conference on Neural Information Processing Systems (NeurIPS)}, 2023.

\bibitem[Ding et~al.(2024{\natexlab{a}})Ding, Fan, Guehring, Gupta, Ha, Huan, Liu, Omidvar-Tehrani, Wang, and Zhou]{ding2024reasoning}
Hao Ding, Ziwei Fan, Ingo Guehring, Gaurav Gupta, Wooseok Ha, Jun Huan, Linbo Liu, Behrooz Omidvar-Tehrani, Shiqi Wang, and Hao Zhou.
\newblock Reasoning and planning with large language models in code development.
\newblock In \emph{Proceedings of the 30th ACM SIGKDD Conference on Knowledge Discovery and Data Mining}, pp.\  6480--6490, 2024{\natexlab{a}}.

\bibitem[Ding et~al.(2024{\natexlab{b}})Ding, Peng, Min, Kaiser, Yang, and Ray]{ding2024semcoder}
Yangruibo Ding, Jinjun Peng, Marcus Min, Gail Kaiser, Junfeng Yang, and Baishakhi Ray.
\newblock Semcoder: Training code language models with comprehensive semantics reasoning.
\newblock \emph{Advances in Neural Information Processing Systems}, 37:\penalty0 60275--60308, 2024{\natexlab{b}}.

\bibitem[Dyer \& Gur-Ari(2022)Dyer and Gur-Ari]{dyer2022minerva}
Ethan Dyer and Guy Gur-Ari.
\newblock Minerva: Solving quantitative reasoning problems with language models.
\newblock \emph{June}, 30:\penalty0 2022, 2022.

\bibitem[Fang et~al.(2025)Fang, Ma, and Wang]{fang2025thinkless}
Gongfan Fang, Xinyin Ma, and Xinchao Wang.
\newblock Thinkless: Llm learns when to think.
\newblock In \emph{The Thirty-ninth Conference on Neural Information Processing Systems (NeurIPS)}, 2025.

\bibitem[Fu et~al.(2025)Fu, Ge, You, Liu, Yuan, Dai, Yan, Yang, and Wang]{fu2025r2r}
Tianyu Fu, Yi~Ge, Yichen You, Enshu Liu, Zhihang Yuan, Guohao Dai, Shengen Yan, Huazhong Yang, and Yu~Wang.
\newblock R2r: Efficiently navigating divergent reasoning paths with small-large model token routing.
\newblock In \emph{The Thirty-ninth Annual Conference on Neural Information Processing Systems}, 2025.

\bibitem[Guo et~al.(2025)Guo, Yang, Zhang, Song, Zhang, Xu, Zhu, Ma, Wang, Bi, et~al.]{guo2025deepseek}
Daya Guo, Dejian Yang, Haowei Zhang, Junxiao Song, Ruoyu Zhang, Runxin Xu, Qihao Zhu, Shirong Ma, Peiyi Wang, Xiao Bi, et~al.
\newblock Deepseek-r1: Incentivizing reasoning capability in llms via reinforcement learning.
\newblock \emph{arXiv preprint arXiv:2501.12948}, 2025.

\bibitem[Han et~al.(2025)Han, Wang, Fang, Zhao, Ma, and Chen]{han2024token}
Tingxu Han, Zhenting Wang, Chunrong Fang, Shiyu Zhao, Shiqing Ma, and Zhenyu Chen.
\newblock Token-budget-aware {LLM} reasoning.
\newblock In Wanxiang Che, Joyce Nabende, Ekaterina Shutova, and Mohammad~Taher Pilehvar (eds.), \emph{Findings of the Association for Computational Linguistics: ACL 2025}, pp.\  24842--24855, Vienna, Austria, July 2025. Association for Computational Linguistics.
\newblock ISBN 979-8-89176-256-5.

\bibitem[Hao et~al.(2024)Hao, Gu, Luo, Liu, Shao, Wang, Xie, Ma, Samavedhi, Gao, Wang, and Hu]{hao2024llm}
Shibo Hao, Yi~Gu, Haotian Luo, Tianyang Liu, Xiyan Shao, Xinyuan Wang, Shuhua Xie, Haodi Ma, Adithya Samavedhi, Qiyue Gao, Zhen Wang, and Zhiting Hu.
\newblock {LLM} reasoners: New evaluation, library, and analysis of step-by-step reasoning with large language models.
\newblock In \emph{First Conference on Language Modeling}, 2024.

\bibitem[He et~al.(2024)He, Luo, and Wang]{he2024st}
Haoyu He, Haozheng Luo, and Qi~R Wang.
\newblock St-moe-bert: A spatial-temporal mixture-of-experts framework for long-term cross-city mobility prediction.
\newblock In \emph{Proceedings of the 2nd ACM SIGSPATIAL International Workshop on Human Mobility Prediction Challenge}, pp.\  10--15, 2024.

\bibitem[He et~al.(2025)He, Luo, Chen, and Wang]{he2025rhythm}
Haoyu He, Haozheng Luo, Yan Chen, and Qi~Wang.
\newblock {RHYTHM}: Reasoning with hierarchical temporal tokenization for human mobility.
\newblock In \emph{The Thirty-ninth Annual Conference on Neural Information Processing Systems}, 2025.

\bibitem[Hou et~al.(2026)Hou, Zhang, Ji, Liu, Qian, Andreas, and Chang]{hou2025thinkprune}
Bairu Hou, Yang Zhang, Jiabao Ji, Yujian Liu, Kaizhi Qian, Jacob Andreas, and Shiyu Chang.
\newblock Thinkprune: Pruning long chain-of-thought of {LLM}s via reinforcement learning.
\newblock \emph{Transactions on Machine Learning Research}, 2026.
\newblock ISSN 2835-8856.

\bibitem[Hu et~al.(2021)Hu, Shen, Wallis, Allen-Zhu, Li, Wang, Wang, and Chen]{hu2021lora}
Edward~J Hu, Yelong Shen, Phillip Wallis, Zeyuan Allen-Zhu, Yuanzhi Li, Shean Wang, Lu~Wang, and Weizhu Chen.
\newblock Lora: Low-rank adaptation of large language models.
\newblock In \emph{The Tenth International Conference on Learning Representations (ICLR)}, 2021.

\bibitem[Hu et~al.(2023)Hu, Yang, Wu, Xu, Chen, and Liu]{hu2024sparse}
Jerry Yao-Chieh Hu, Donglin Yang, Dennis Wu, Chenwei Xu, Bo-Yu Chen, and Han Liu.
\newblock On sparse modern hopfield model.
\newblock In \emph{The Thirty-seventh Conference on Neural Information Processing Systems (NeurIPS)}, 2023.

\bibitem[Hu et~al.(2024)Hu, Chang, Luo, Chen, Li, Wang, and Liu]{hu2024outlier}
Jerry Yao-Chieh Hu, Pei-Hsuan Chang, Robin Luo, Hong-Yu Chen, Weijian Li, Wei-Po Wang, and Han Liu.
\newblock Outlier-efficient hopfield layers for large transformer-based models.
\newblock In \emph{The Forty-first International Conference on Machine Learning (ICML)}, 2024.

\bibitem[{Hugging Face}(2025)]{openr1}
{Hugging Face}.
\newblock Open r1: A fully open reproduction of deepseek-r1, January 2025.
\newblock URL \url{https://github.com/huggingface/open-r1}.

\bibitem[Hurst et~al.(2024)Hurst, Lerer, Goucher, Perelman, Ramesh, Clark, Ostrow, Welihinda, Hayes, Radford, et~al.]{hurst2024gpt}
Aaron Hurst, Adam Lerer, Adam~P Goucher, Adam Perelman, Aditya Ramesh, Aidan Clark, AJ~Ostrow, Akila Welihinda, Alan Hayes, Alec Radford, et~al.
\newblock Gpt-4o system card.
\newblock \emph{arXiv preprint arXiv:2410.21276}, 2024.

\bibitem[Jaech et~al.(2024)Jaech, Kalai, Lerer, Richardson, El-Kishky, Low, Helyar, Madry, Beutel, Carney, et~al.]{jaech2024openai}
Aaron Jaech, Adam Kalai, Adam Lerer, Adam Richardson, Ahmed El-Kishky, Aiden Low, Alec Helyar, Aleksander Madry, Alex Beutel, Alex Carney, et~al.
\newblock Openai o1 system card.
\newblock \emph{arXiv preprint arXiv:2412.16720}, 2024.

\bibitem[Jain et~al.(2025)Jain, Han, Gu, Li, Yan, Zhang, Wang, Solar-Lezama, Sen, and Stoica]{jain2024livecodebench}
Naman Jain, King Han, Alex Gu, Wen-Ding Li, Fanjia Yan, Tianjun Zhang, Sida Wang, Armando Solar-Lezama, Koushik Sen, and Ion Stoica.
\newblock Livecodebench: Holistic and contamination free evaluation of large language models for code.
\newblock In \emph{The Thirteenth International Conference on Learning Representations}, 2025.

\bibitem[Jiang et~al.(2025{\natexlab{a}})Jiang, Luo, Pang, Li, Qi, Li, Yang, Lin, Li, Xu, et~al.]{jiang2025learning}
Eric~Hanchen Jiang, Haozheng Luo, Shengyuan Pang, Xiaomin Li, Zhenting Qi, Hengli Li, Cheng-Fu Yang, Zongyu Lin, Xinfeng Li, Hao Xu, et~al.
\newblock Learning to rank chain-of-thought: An energy-based approach with outcome supervision.
\newblock \emph{arXiv preprint arXiv:2505.14999}, 2025{\natexlab{a}}.

\bibitem[Jiang et~al.(2025{\natexlab{b}})Jiang, Li, and Ferraro]{jiang2025drp}
Yuxuan Jiang, Dawei Li, and Frank Ferraro.
\newblock Drp: Distilled reasoning pruning with skill-aware step decomposition for efficient large reasoning models.
\newblock \emph{arXiv preprint arXiv:2505.13975}, 2025{\natexlab{b}}.

\bibitem[Ke et~al.()Ke, Jiao, Ming, Nguyen, Xu, Long, Li, Qin, Wang, silvio savarese, Xiong, and Joty]{ke2025survey}
Zixuan Ke, Fangkai Jiao, Yifei Ming, Xuan-Phi Nguyen, Austin Xu, Do~Xuan Long, Minzhi Li, Chengwei Qin, PeiFeng Wang, silvio savarese, Caiming Xiong, and Shafiq Joty.
\newblock A survey of frontiers in {LLM} reasoning: Inference scaling, learning to reason, and agentic systems.
\newblock \emph{Transactions on Machine Learning Research}.
\newblock ISSN 2835-8856.
\newblock Survey Certification.

\bibitem[Kumar et~al.(2025)Kumar, Roh, Naseh, Karpinska, Iyyer, Houmansadr, and Bagdasarian]{kumar2025overthinking}
Abhinav Kumar, Jaechul Roh, Ali Naseh, Marzena Karpinska, Mohit Iyyer, Amir Houmansadr, and Eugene Bagdasarian.
\newblock Overthinking: Slowdown attacks on reasoning llms.
\newblock \emph{arXiv preprint arXiv:2502.02542}, 2025.

\bibitem[Li et~al.(2025)Li, Dong, Ma, Zhang, Jia, and Sui]{li2025selfbudgeter}
Zheng Li, Qingxiu Dong, Jingyuan Ma, Di~Zhang, Kai Jia, and Zhifang Sui.
\newblock Selfbudgeter: Adaptive token allocation for efficient {LLM} reasoning, 2025.

\bibitem[Lightman et~al.(2024)Lightman, Kosaraju, Burda, Edwards, Baker, Lee, Leike, Schulman, Sutskever, and Cobbe]{lightman2023lets}
Hunter Lightman, Vineet Kosaraju, Yuri Burda, Harrison Edwards, Bowen Baker, Teddy Lee, Jan Leike, John Schulman, Ilya Sutskever, and Karl Cobbe.
\newblock Let's verify step by step.
\newblock In \emph{The Twelfth International Conference on Learning Representations}, 2024.

\bibitem[Liu et~al.(2025)Liu, Zheng, Sun, Peng, Dong, Sha, Cui, Wang, and He]{liu2025thought}
Yule Liu, Jingyi Zheng, Zhen Sun, Zifan Peng, Wenhan Dong, Zeyang Sha, Shiwen Cui, Weiqiang Wang, and Xinlei He.
\newblock Thought manipulation: External thought can be efficient for large reasoning models.
\newblock \emph{arXiv preprint arXiv:2504.13626}, 2025.

\bibitem[Loshchilov \& Hutter(2019)Loshchilov and Hutter]{loshchilov2017decoupled}
Ilya Loshchilov and Frank Hutter.
\newblock Decoupled weight decay regularization.
\newblock In \emph{International Conference on Learning Representations}, 2019.

\bibitem[Luo et~al.(2025{\natexlab{a}})Luo, Sun, Xu, Zhao, Lou, Tao, Geng, Lin, Chen, Tang, and Zhang]{luo2025wizardmath}
Haipeng Luo, Qingfeng Sun, Can Xu, Pu~Zhao, Jian-Guang Lou, Chongyang Tao, Xiubo Geng, Qingwei Lin, Shifeng Chen, Yansong Tang, and Dongmei Zhang.
\newblock Wizardmath: Empowering mathematical reasoning for large language models via reinforced evol-instruct.
\newblock In \emph{The Thirteenth International Conference on Learning Representations}, 2025{\natexlab{a}}.

\bibitem[Luo et~al.(2025{\natexlab{b}})Luo, Qiu, Su, Zhou, Mehta, Ye, Hu, and Liu]{luo2025fast}
Haozheng Luo, Chenghao Qiu, Maojiang Su, Zhihan Zhou, Zoe Mehta, Guo Ye, Jerry Yao-Chieh Hu, and Han Liu.
\newblock Fast and low-cost genomic foundation models via outlier removal.
\newblock In \emph{Forty-second International Conference on Machine Learning}, 2025{\natexlab{b}}.

\bibitem[Luo et~al.(2025{\natexlab{c}})Luo, Yu, Zhang, Li, Hu, Chen, Wang, Xing, and Liu]{luo2024decoupled}
Haozheng Luo, Jiahao Yu, Wenxin Zhang, Jialong Li, Jerry Yao-Chieh Hu, Yan Chen, Binghui Wang, Xinyu Xing, and Han Liu.
\newblock Decoupled alignment for robust plug-and-play adaptation.
\newblock In \emph{The Impact of Memorization on Trustworthy Foundation Models: ICML 2025 Workshop}, 2025{\natexlab{c}}.

\bibitem[Ma et~al.(2025)Ma, Wan, Yu, Fang, and Wang]{ma2025cot}
Xinyin Ma, Guangnian Wan, Runpeng Yu, Gongfan Fang, and Xinchao Wang.
\newblock Cot-valve: Length-compressible chain-of-thought tuning.
\newblock In \emph{Submitted to ACL Rolling Review - February 2025}, 2025.
\newblock under review.

\bibitem[Martins \& Astudillo(2016)Martins and Astudillo]{martins2016softmax}
Andre Martins and Ramon Astudillo.
\newblock From softmax to sparsemax: A sparse model of attention and multi-label classification.
\newblock In \emph{International conference on machine learning}, pp.\  1614--1623. PMLR, 2016.

\bibitem[of~America(2024)]{aime_2024}
Mathematical~Association of~America.
\newblock American invitational mathematics examination 2024, 2024.
\newblock Official competition problems.

\bibitem[Ouyang et~al.(2022)Ouyang, Wu, Jiang, Almeida, Wainwright, Mishkin, Zhang, Agarwal, Slama, Ray, Schulman, Hilton, Kelton, Miller, Simens, Askell, Welinder, Christiano, Leike, and Lowe]{ouyang2022training}
Long Ouyang, Jeff Wu, Xu~Jiang, Diogo Almeida, Carroll~L. Wainwright, Pamela Mishkin, Chong Zhang, Sandhini Agarwal, Katarina Slama, Alex Ray, John Schulman, Jacob Hilton, Fraser Kelton, Luke Miller, Maddie Simens, Amanda Askell, Peter Welinder, Paul Christiano, Jan Leike, and Ryan Lowe.
\newblock Training language models to follow instructions with human feedback.
\newblock In \emph{The Thirty-sixth Conference on Neural Information Processing Systems (NeurIPS)}, 2022.

\bibitem[Pan et~al.(2024)Pan, Luo, Li, and Liu]{pan2024conv}
Zhenyu Pan, Haozheng Luo, Manling Li, and Han Liu.
\newblock Conv-coa: Improving open-domain question answering in large language models via conversational chain-of-action.
\newblock \emph{arXiv preprint arXiv:2405.17822}, 2024.

\bibitem[Pan et~al.(2025)Pan, Luo, Li, and Liu]{pan2025chainofaction}
Zhenyu Pan, Haozheng Luo, Manling Li, and Han Liu.
\newblock Chain-of-action: Faithful and multimodal question answering through large language models.
\newblock In \emph{The Thirteenth International Conference on Learning Representations}, 2025.

\bibitem[Press et~al.(2023)Press, Zhang, Min, Schmidt, Smith, and Lewis]{press2022measuring}
Ofir Press, Muru Zhang, Sewon Min, Ludwig Schmidt, Noah Smith, and Mike Lewis.
\newblock Measuring and narrowing the compositionality gap in language models.
\newblock In Houda Bouamor, Juan Pino, and Kalika Bali (eds.), \emph{Findings of the Association for Computational Linguistics: EMNLP 2023}, pp.\  5687--5711, Singapore, December 2023. Association for Computational Linguistics.

\bibitem[Rafailov et~al.(2023)Rafailov, Sharma, Mitchell, Ermon, Manning, and Finn]{rafailov2023direct}
Rafael Rafailov, Archit Sharma, Eric Mitchell, Stefano Ermon, Christopher~D. Manning, and Chelsea Finn.
\newblock Direct preference optimization: Your language model is secretly a reward model.
\newblock In \emph{The Thirty-seventh Conference on Neural Information Processing Systems (NeurIPS)}, 2023.

\bibitem[Ramesh et~al.(2024)Ramesh, Hu, Chaimalas, Mehta, Sessa, Bou~Ammar, and Bogunovic]{ramesh2024group}
Shyam~Sundhar Ramesh, Yifan Hu, Iason Chaimalas, Viraj Mehta, Pier~Giuseppe Sessa, Haitham Bou~Ammar, and Ilija Bogunovic.
\newblock Group robust preference optimization in reward-free rlhf.
\newblock \emph{Advances in Neural Information Processing Systems}, 37:\penalty0 37100--37137, 2024.

\bibitem[Rastogi et~al.(2025)Rastogi, Jiang, Lo, Berrada, Lample, Rute, Barmentlo, Yadav, Khandelwal, Chandu, et~al.]{rastogi2025magistral}
Abhinav Rastogi, Albert~Q Jiang, Andy Lo, Gabrielle Berrada, Guillaume Lample, Jason Rute, Joep Barmentlo, Karmesh Yadav, Kartik Khandelwal, Khyathi~Raghavi Chandu, et~al.
\newblock Magistral.
\newblock \emph{arXiv preprint arXiv:2506.10910}, 2025.

\bibitem[Shao et~al.(2024)Shao, Wang, Zhu, Xu, Song, Bi, Zhang, Zhang, Li, Wu, et~al.]{shao2024deepseekmath}
Zhihong Shao, Peiyi Wang, Qihao Zhu, Runxin Xu, Junxiao Song, Xiao Bi, Haowei Zhang, Mingchuan Zhang, YK~Li, Yang Wu, et~al.
\newblock Deepseekmath: Pushing the limits of mathematical reasoning in open language models.
\newblock \emph{arXiv preprint arXiv:2402.03300}, 2024.

\bibitem[Sui et~al.(2025)Sui, Chuang, Wang, Zhang, Zhang, Yuan, Liu, Wen, Zhong, Chen, et~al.]{sui2025stop}
Yang Sui, Yu-Neng Chuang, Guanchu Wang, Jiamu Zhang, Tianyi Zhang, Jiayi Yuan, Hongyi Liu, Andrew Wen, Shaochen Zhong, Hanjie Chen, et~al.
\newblock Stop overthinking: A survey on efficient reasoning for large language models.
\newblock \emph{Submitted to Transactions on Machine Learning Research}, 2025.

\bibitem[Tan \& Pan(2025)Tan and Pan]{tan2025gtpo}
Hongze Tan and Jianfei Pan.
\newblock Gtpo and grpo-s: Token and sequence-level reward shaping with policy entropy.
\newblock \emph{arXiv preprint arXiv:2508.04349}, 2025.

\bibitem[Team et~al.(2023)Team, Anil, Borgeaud, Alayrac, Yu, Soricut, Schalkwyk, Dai, Hauth, Millican, et~al.]{team2023gemini}
Gemini Team, Rohan Anil, Sebastian Borgeaud, Jean-Baptiste Alayrac, Jiahui Yu, Radu Soricut, Johan Schalkwyk, Andrew~M Dai, Anja Hauth, Katie Millican, et~al.
\newblock Gemini: a family of highly capable multimodal models.
\newblock \emph{arXiv preprint arXiv:2312.11805}, 2023.

\bibitem[Touvron et~al.(2023)Touvron, Martin, Stone, Albert, Almahairi, Babaei, Bashlykov, Batra, Bhargava, Bhosale, et~al.]{touvron2023llama}
Hugo Touvron, Louis Martin, Kevin Stone, Peter Albert, Amjad Almahairi, Yasmine Babaei, Nikolay Bashlykov, Soumya Batra, Prajjwal Bhargava, Shruti Bhosale, et~al.
\newblock Llama 2: Open foundation and fine-tuned chat models.
\newblock \emph{arXiv preprint arXiv:2307.09288}, 2023.

\bibitem[Wang(2024)]{wang2024reducing}
Libo Wang.
\newblock Reducing reasoning costs - the path of optimization for chain of thought via sparse attention mechanism.
\newblock In \emph{Muslims in ML Workshop co-located with NeurIPS 2024}, 2024.

\bibitem[Wang et~al.(2024)Wang, Bai, Tan, Wang, Fan, Bai, Chen, Liu, Wang, Ge, et~al.]{wang2024qwen2}
Peng Wang, Shuai Bai, Sinan Tan, Shijie Wang, Zhihao Fan, Jinze Bai, Keqin Chen, Xuejing Liu, Jialin Wang, Wenbin Ge, et~al.
\newblock Qwen2-vl: Enhancing vision-language model's perception of the world at any resolution.
\newblock \emph{arXiv preprint arXiv:2409.12191}, 2024.

\bibitem[Wang et~al.(2025)Wang, Yu, Gao, Zheng, Liu, Lu, Dang, Chen, Yang, Zhang, Liu, Yang, Zhao, Yue, Song, Yu, Huang, and Lin]{wang2025beyond}
Shenzhi Wang, Le~Yu, Chang Gao, Chujie Zheng, Shixuan Liu, Rui Lu, Kai Dang, Xiong-Hui Chen, Jianxin Yang, Zhenru Zhang, Yuqiong Liu, An~Yang, Andrew Zhao, Yang Yue, Shiji Song, Bowen Yu, Gao Huang, and Junyang Lin.
\newblock Beyond the 80/20 rule: High-entropy minority tokens drive effective reinforcement learning for {LLM} reasoning.
\newblock In \emph{The Thirty-ninth Annual Conference on Neural Information Processing Systems}, 2025.

\bibitem[Wei et~al.(2022)Wei, Wang, Schuurmans, Bosma, Xia, Chi, Le, Zhou, et~al.]{wei2022chain}
Jason Wei, Xuezhi Wang, Dale Schuurmans, Maarten Bosma, Fei Xia, Ed~Chi, Quoc~V Le, Denny Zhou, et~al.
\newblock Chain-of-thought prompting elicits reasoning in large language models.
\newblock In \emph{The Thirty-sixth Conference on Neural Information Processing Systems (NeurIPS)}, 2022.

\bibitem[Wu et~al.(2024)Wu, Hu, Li, Chen, and Liu]{wu2023stanhop}
Dennis Wu, Jerry Yao-Chieh Hu, Weijian Li, Bo-Yu Chen, and Han Liu.
\newblock {ST}anhop: Sparse tandem hopfield model for memory-enhanced time series prediction.
\newblock In \emph{The Twelfth International Conference on Learning Representations (ICLR)}, 2024.

\bibitem[Xia et~al.(2025{\natexlab{a}})Xia, Leong, Wang, Li, and Li]{xia2025tokenskip}
Heming Xia, Chak~Tou Leong, Wenjie Wang, Yongqi Li, and Wenjie Li.
\newblock {T}oken{S}kip: Controllable chain-of-thought compression in {LLM}s.
\newblock In Christos Christodoulopoulos, Tanmoy Chakraborty, Carolyn Rose, and Violet Peng (eds.), \emph{Proceedings of the 2025 Conference on Empirical Methods in Natural Language Processing}, pp.\  3351--3363, Suzhou, China, November 2025{\natexlab{a}}. Association for Computational Linguistics.
\newblock ISBN 979-8-89176-332-6.

\bibitem[Xia et~al.(2025{\natexlab{b}})Xia, Shen, Wang, Liu, Sun, Wu, Hu, and Xu]{xia2025leetcodedataset}
Yunhui Xia, Wei Shen, Yan Wang, Jason~Klein Liu, Huifeng Sun, Siyue Wu, Jian Hu, and Xiaolong Xu.
\newblock Leetcodedataset: A temporal dataset for robust evaluation and efficient training of code llms.
\newblock \emph{arXiv preprint arXiv:2504.14655}, 2025{\natexlab{b}}.

\bibitem[Xiao et~al.(2024)Xiao, Tian, Chen, Han, and Lewis]{xiao2023efficient}
Guangxuan Xiao, Yuandong Tian, Beidi Chen, Song Han, and Mike Lewis.
\newblock Efficient streaming language models with attention sinks.
\newblock In \emph{The Twelfth International Conference on Learning Representations}, 2024.

\bibitem[Xu et~al.(2025{\natexlab{a}})Xu, Xie, Zhao, and He]{xu2025chain}
Silei Xu, Wenhao Xie, Lingxiao Zhao, and Pengcheng He.
\newblock Chain of draft: Thinking faster by writing less.
\newblock \emph{arXiv preprint arXiv:2502.18600}, 2025{\natexlab{a}}.

\bibitem[Xu et~al.(2025{\natexlab{b}})Xu, Xu, Xiao, Chen, Yan, ZHANG, Diao, Yang, and Wang]{xu2025ugphysics}
Xin Xu, Qiyun Xu, Tong Xiao, Tianhao Chen, Yuchen Yan, Jiaxin ZHANG, Shizhe Diao, Can Yang, and Yang Wang.
\newblock {UGP}hysics: A comprehensive benchmark for undergraduate physics reasoning with large language models.
\newblock In \emph{Forty-second International Conference on Machine Learning}, 2025{\natexlab{b}}.

\bibitem[Yang et~al.(2024)Yang, Zhang, Hui, Gao, Yu, Li, Liu, Tu, Zhou, Lin, et~al.]{yang2024qwen2}
An~Yang, Beichen Zhang, Binyuan Hui, Bofei Gao, Bowen Yu, Chengpeng Li, Dayiheng Liu, Jianhong Tu, Jingren Zhou, Junyang Lin, et~al.
\newblock Qwen2. 5-math technical report: Toward mathematical expert model via self-improvement.
\newblock \emph{arXiv preprint arXiv:2409.12122}, 2024.

\bibitem[Yang et~al.(2025)Yang, Zhang, Jain, Cao, Yuan, Chen, Jia, and Netravali]{yang2025less}
Lijie Yang, Zhihao Zhang, Arti Jain, Shijie Cao, Baihong Yuan, Yiwei Chen, Zhihao Jia, and Ravi Netravali.
\newblock Less is more: Training-free sparse attention with global locality for efficient reasoning.
\newblock \emph{arXiv preprint arXiv:2508.07101}, 2025.

\bibitem[Yao et~al.(2023)Yao, Zhao, Yu, Du, Shafran, Narasimhan, and Cao]{yao2023react}
Shunyu Yao, Jeffrey Zhao, Dian Yu, Nan Du, Izhak Shafran, Karthik Narasimhan, and Yuan Cao.
\newblock React: Synergizing reasoning and acting in language models.
\newblock In \emph{International Conference on Learning Representations (ICLR)}, 2023.

\bibitem[Yi et~al.(2025)Yi, Wang, and Li]{yi2025shorterbetter}
Jingyang Yi, Justin Wang, and Sida Li.
\newblock Shorterbetter: Guiding reasoning models to find optimal inference length for efficient reasoning.
\newblock In \emph{The Thirty-ninth Annual Conference on Neural Information Processing Systems}, 2025.

\bibitem[Zhao et~al.(2025{\natexlab{a}})Zhao, Yan, Shen, Xu, Zhang, Song, Shao, Lu, Xiao, and Zhuang]{zhao2025let}
Haoran Zhao, Yuchen Yan, Yongliang Shen, Haolei Xu, Wenqi Zhang, Kaitao Song, Jian Shao, Weiming Lu, Jun Xiao, and Yueting Zhuang.
\newblock Let {LRM}s break free from overthinking via self-braking tuning.
\newblock In \emph{The Thirty-ninth Annual Conference on Neural Information Processing Systems}, 2025{\natexlab{a}}.

\bibitem[Zhao et~al.(2025{\natexlab{b}})Zhao, Zhang, Xie, Goyal, Kawaguchi, and Shieh]{zhao2025understanding}
Yiran Zhao, Wenxuan Zhang, Yuxi Xie, Anirudh Goyal, Kenji Kawaguchi, and Michael Shieh.
\newblock Understanding and enhancing safety mechanisms of {LLM}s via safety-specific neuron.
\newblock In \emph{The Thirteenth International Conference on Learning Representations}, 2025{\natexlab{b}}.

\bibitem[Zhou et~al.(2025)Zhou, Riley, Kautsar, Wu, Egan, Hofmeyr, Goldhaber-Gordon, Yu, Ho, Liu, et~al.]{zhou2025genomeocean}
Zhihan Zhou, Robert Riley, Satria Kautsar, Weimin Wu, Rob Egan, Steven Hofmeyr, Shira Goldhaber-Gordon, Mutian Yu, Harrison Ho, Fengchen Liu, et~al.
\newblock Genomeocean: An efficient genome foundation model trained on large-scale metagenomic assemblies.
\newblock \emph{bioRxiv}, pp.\  2025--01, 2025.

\end{thebibliography}
